\documentclass[10pt]{article}
\usepackage[preprint]{tmlr/tmlr}

\usepackage[utf8]{inputenc}
\usepackage[T1]{fontenc}
\usepackage[final]{microtype}
\usepackage{fontawesome}
\usepackage{dsfont}
\usepackage{stmaryrd}
\usepackage{verbatim}
\usepackage{pifont}

\usepackage[pagebackref=true, hidelinks]{hyperref}
\usepackage{natbib}
\usepackage{url}

\usepackage[final,pdftex]{graphicx}
\usepackage{float}
\usepackage{caption}
\usepackage{tcolorbox}
\usepackage{booktabs,array}
\usepackage{tabularx}
\usepackage{makecell}
\usepackage{multirow}
\usepackage{tikz}
\usetikzlibrary{shapes.geometric,positioning}

\usepackage[frozencache,cachedir=minted-output]{minted}

\usepackage[boxed]{algorithm2e}
\usepackage[noend]{algcompatible}

\usepackage[amsthm, thmmarks]{ntheorem}
\usepackage{amsfonts,amssymb,amsmath}
\usepackage[noabbrev,nameinlink,capitalise,compress]{cleveref}
\usepackage{mathtools}
\usepackage{bm}
\usepackage{mathabx}
\usepackage{thmtools,thm-restate}
\usepackage{resizegather}
\usepackage{nicefrac}

\usepackage{enumitem}
\usepackage{setspace}
\usepackage{regexpatch}
\newcommand{\cmark}{\textcolor{green!80!black}{\ding{51}}}
\newcommand{\xmark}{\textcolor{red!80!black}{\ding{55}}}

\newcommand{\defeq}{\vcentcolon=}
\newcommand{\el}{\mathcal{L}}

\newcommand{\softmax}{\operatorname{softmax}}
\newcommand{\relu}{\operatorname{relu}}
\newcommand{\bregman}{\operatorname{bregman}}
\newcommand{\kl}{D_\mathrm{KL}}

\renewcommand{\phi}{\varphi}

\newcommand{\set}[1]{\mathsf{#1}}

\newcommand{\half}{\tfrac{1}{2}}

\newcommand{\R}{\mathbb{R}}

\newcommand{\rank}{\operatorname{rank}}
\newcommand{\srank}{\operatorname{rank_{stable}}}
\newcommand{\trace}{\operatorname{tr}}
\newcommand{\diag}{\operatorname{diag}}

\newcommand{\abs}[1]{\vert {#1} \vert}
\newcommand{\norm}[1]{\Vert {#1} \Vert}

\newcommand{\uniform}{\textsc{uniform}}

\def\vf{{\bm{f}}}
\def\vg{{\bm{g}}}
\def\vh{{\bm{h}}}

\def\vu{{\bm{u}}}
\def\vv{{\bm{v}}}
\def\vw{{\bm{w}}}
\def\vx{{\bm{x}}}
\def\vy{{\bm{y}}}

\def\mD{{\bm{D}}}

\def\mM{{\bm{M}}}

\def\mO{{\bm{O}}}
\def\mP{{\bm{P}}}

\def\mW{{\bm{W}}}

\renewcommand*{\backrefalt}[4]{\ifcase #1 No citations. \or Cited on page #2. \else Cited on pages #2. \fi}

\fancypagestyle{empty}{\fancyhf{}}
\fancypagestyle{plain}{\fancyfoot[C]{\sffamily\selectfont\thepage}}
\newcommand{\captiontitle}[1]{\textsf{\textbf{#1}}}

\makeatletter
\renewcommand\paragraph{\@startsection{paragraph}{4}{\z@}{0ex \@plus1ex \@minus.2ex}{-1em}{\bfseries\sffamily}}
\makeatother

\makeatletter
\xpatchcmd \thmt@restatable{\thmt@toks{}}
{\def\thmt@tmp@restatename{#3}\thmt@toks{}}{}{\fail}

\renewtheoremstyle{plain}{}
  {\item[\ifcsname thmt@tmp@restatename\endcsname
        \ifthmt@thisistheone\hskip\labelsep\mbox{\hyperref[proof:\thmt@tmp@restatename]{\bf\sffamily ##1\ ##2}} \fi
  \else \hskip\labelsep\bf\sffamily ##1\ ##2 \fi \normalfont\sffamily(##3)\theorem@separator]}
\makeatother

\theoremstyle{plain}
\theorembodyfont{\normalfont}

\theoremstyle{plain}
\theorembodyfont{\normalfont}
\newtheorem{fact}{Fact}

\theoremstyle{plain}
\theorembodyfont{\normalfont}
\newtheorem{prescription}{Prescription}
\crefname{prescription}{Prescription}{Prescriptions}

\theoremstyle{plain}
\theorembodyfont{\normalfont}

\theoremstyle{plain}
\theorembodyfont{\normalfont}

\theoremstyle{plain}
\theorembodyfont{\normalfont}

\theoremstyle{plain}
\theorembodyfont{\normalfont}

\theoremstyle{plain}
\theorembodyfont{\normalfont}

\theoremstyle{plain}
\theorembodyfont{\normalfont}
\newtheorem{assumption}{Assumption}
\crefname{assumption}{Assumption}{Assumptions}

\theoremstyle{plain}
\theorembodyfont{\normalfont}
\newtheorem{definition}{Definition}

\theoremstyle{plain}
\theorembodyfont{\normalfont}
\newtheorem{example}{Example}

\newcommand{\WCOMMENT}[1]{\hfill\begin{minipage}{20em}\COMMENT{#1}\end{minipage}}
\SetAlCapSkip{1em}
\SetKwComment{Comment}{\small$\#$ }{}

\algrenewcommand\algorithmicindent{2.0em}

\algrenewcommand\alglinenumber[1]{\sffamily\footnotesize #1.}

\algblockdefx[DEF]{DEF}{ENDDEF}%
    [1][function]{\textbf{def} \texttt{#1()}:}
    
\makeatletter
\ifthenelse{\equal{\ALG@noend}{t}}{\algtext*{ENDDEF}}
\makeatother

\makeatletter
\xpatchcmd{\algorithmic}{\labelsep 0.5em}{\labelsep 1.0em}{\typeout{Success!}}{\typeout{Oh dear!}}
\makeatother

\title{Automatic Gradient Descent:\\Deep Learning without Hyperparameters}

\newcommand{\authspace}{\hspace{3.19em}}
\newcommand{\auth}[2]{\begin{tabular}{@{}l@{}}{#1}\\\normalfont{#2}\end{tabular}}

\author{\sffamily\auth{Jeremy Bernstein$^\star$}{MIT}\authspace\auth{\hspace{-5pt}Chris Mingard$^\star$}{\hspace{-5pt}U.\ Oxford}\authspace\auth{Kevin Huang}{U.\ Washington}\authspace\auth{Navid Azizan}{MIT} \authspace\auth{Yisong Yue}{Caltech}}

\begin{document}

\maketitle
\thispagestyle{empty}

\vspace{-5ex}
\hfill$\star$ denotes equal contribution.\\

\begin{abstract}
    The architecture of a deep neural network is defined explicitly in terms of the number of layers, the width of each layer and the general network topology. Existing optimisation frameworks neglect this information in favour of implicit architectural information (e.g.~second-order methods) or architecture-agnostic distance functions (e.g.~mirror descent). Meanwhile, the most popular optimiser in practice---Adam---is based on heuristics. This paper builds a new framework for deriving optimisation algorithms that explicitly leverage neural architecture. The theory extends mirror descent to non-convex composite objective functions: the idea is to transform a Bregman divergence to account for the non-linear structure of neural architecture. Working through the details for deep fully-connected networks yields \textit{automatic gradient descent}: a first-order optimiser without any hyperparameters. Automatic gradient descent trains both fully-connected and convolutional networks out-of-the-box and at ImageNet scale. A PyTorch implementation is available at \url{https://github.com/jxbz/agd} and also in \cref{app:pytorch}. Overall, the paper supplies a rigorous theoretical foundation for a next-generation of architecture-dependent optimisers that work automatically and without hyperparameters.
\end{abstract}

\sffamily
\setstretch{0}
\tableofcontents
\setstretch{1}
\normalfont

\begin{figure}
    \centering
    \includegraphics[viewport=0 0 16.32cm 6.69cm]{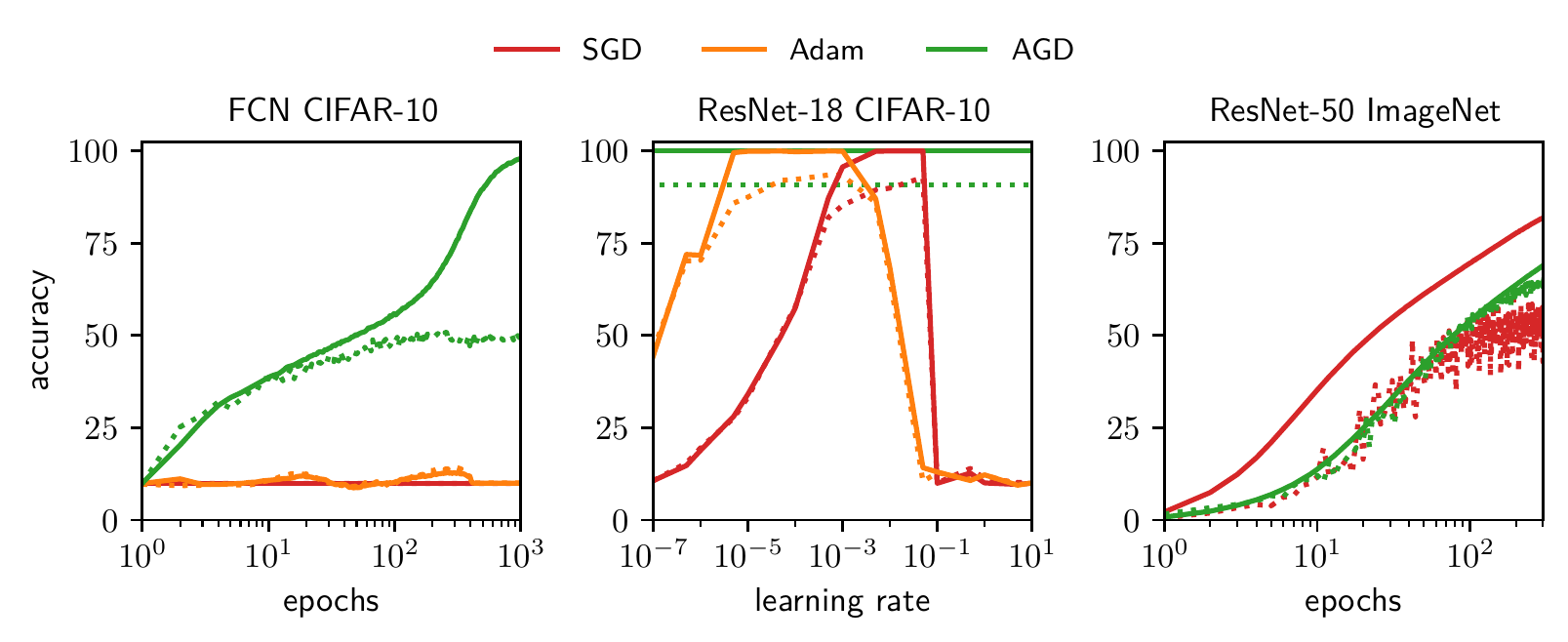}
    \caption{\captiontitle{Automatic gradient descent trains neural networks reliably without hyperparameters.} Solid lines show train accuracy and dotted lines show test accuracy. The networks are unregularised with biases and affine parameters disabled, as these features are not yet supported by AGD. In the \captiontitle{left panel}---unlike AGD---Adam and SGD failed to train a 32-layer fully-connected network on CIFAR-10 with their default learning rates of 0.001 for Adam and 0.1 for SGD. The \captiontitle{middle panel} displays a learning rate grid search for ResNet-18 trained on CIFAR-10. AGD attained performance comparable to the best tuned performance of Adam and SGD. In the \captiontitle{right panel}, AGD trained ResNet-50 on ImageNet to a top-1 test accuracy of 65.5\%. The ImageNet baseline is SGD with a learning rate of 0.1 and no learning rate decay schedule.}
    \label{fig:showcase}
\end{figure}

\section{Introduction}

Automatic differentiation has contributed to the rapid pace of innovation in the field of deep learning. Software packages such as PyTorch \citep{pytorch} and Theano \citep{theano} have advanced a programming paradigm where the user (1) defines a neural network architecture by composing differentiable operators and (2) supplies training data. The package then automatically computes the gradient of the error on the training data via recursive application of the chain rule. At this point, the user must become involved again by (3) selecting one of numerous optimisation algorithms and (4) manually tuning its hyperparameters: in particular, the initial learning rate and the learning rate decay schedule \citep{Goodfellow-et-al-2016}.

But manually tuning hyperparameters is irksome. An abundance of hyperparameters makes it difficult to rank the performance of different deep learning algorithms \citep{Lucic2017AreGC,crowded_valley} and difficult to reproduce results in the literature \citep{deeprlmatters}. Hyperparameters confound our efforts to build a scientific understanding of generalisation in deep learning \citep{jiang2019fantastic, my-margin}. And, when training neural networks at the largest scale, in pursuit of stronger forms of artificial intelligence, hyperparameter grid search can rack up millions of dollars in compute costs \citep{Sharir2020TheCO}. 

Are hyperparameters just a fact of life? The thesis of this paper is that \textit{no: they are not}. Deep learning involves fitting a known function to known data via minimising a known objective. If we could characterise these components both individually and in how they interact, then---in principle---there should be no leftover degrees of freedom to be tuned \citep{tutorial}. Taking this idea and running with it leads to \textit{automatic gradient descent} (AGD): a neural network optimiser without any hyperparameters. AGD is complementary to automatic differentiation and could help to automate general machine learning workflows.

Two existing tools are central to our derivation, and it is their novel combination that presents the main theoretical contribution of this paper. First, a classic tool from convex analysis known as the \textit{Bregman divergence} \citep{bregman1967relaxation,bregman} is used to characterise how the neural network interacts with the loss function. And second, a tool called \textit{deep relative trust} \citep{my-fromage} is used to characterise the highly non-linear interaction between the weights and the network output. With these tools in hand, we can apply the \textit{majorise-minimise meta-algorithm} \citep{mm} to derive an optimiser explicitly tailored to deep network objective functions. To summarise, the derivation of AGD follows three main steps:

\begin{enumerate}[label=Step \arabic*:, leftmargin=*, font=\sffamily]
    \item \textsf{Functional expansion}. We use a \textit{Bregman divergence} to express the linearisation error of the objective function $\el(\vw)$ in terms of the functional perturbation $\Delta \vf$ to the network $\vf$.
    \item \textsf{Architectural perturbation bounds.} We use \textit{deep relative trust} to relate the size and structure of the weight perturbation $\Delta \vw$ to the size of the induced functional perturbation $\Delta \vf$.
    \item \textsf{Majorise-minimise.} We substitute deep relative trust into the Bregman divergence to obtain an explicitly architecture-dependent majorisation. Minimising with respect to $\Delta \vw$ yields an optimiser.
\end{enumerate}

\paragraph{Summary of contributions} This paper derives automatic gradient descent (AGD) by applying the majorise-minimise meta-algorithm to deep network objective functions. AGD trains all tested network architectures without hyperparameters, and scales to deep networks such as ResNet-50 and large datasets such as ImageNet. AGD trains out-of-the-box even when Adam and SGD fail to train with their default hyperparameters.

\begin{table}
    \centering
    \begin{tabularx}{\textwidth}{lXccccc}
        \toprule
        \textbf{Optimiser} &
        \textbf{Reference} & 
        \makecell{\textbf{Hyperparameter}\\\textbf{Free}} &
        \makecell{\textbf{Width}\\\textbf{Scaling}} & \makecell{\textbf{Depth}\\\textbf{Scaling}} &  
        \makecell{\textbf{Automatic}\\\textbf{Schedule}} & 
        \makecell{\textbf{Memory}\\\textbf{Cost}} \\ 
        \midrule
        Adam    & $\mathrlap{\text{\citet{kingma_adam:_2015}}}$ & \xmark & \xmark & \xmark & \xmark & $3 \times \#$weights\\
        SGD + mom.     & $\mathrlap{\text{\citet{bottou}}}$ & \xmark & \xmark & \xmark & \xmark & $2\times\#$weights\\
        SGD + muP     & $\mathrlap{\text{\citet{Yang2021TensorPI}}}$ & \xmark & \cmark & \xmark & \xmark & $1\times\#$weights\\
        AGD     & this paper                & \cmark & \cmark & \cmark & \cmark & $1\times\#$weights\\ 
        \bottomrule
    \end{tabularx}
    \caption{\captiontitle{Comparing practical optimisers.} Adam and momentum-SGD employ running estimates of gradient statistics and thereby use more memory than AGD. In addition, Adam and SGD do not provide guidance on scaling hyperparameters with network architecture, although muP fixes this for the case of width scaling.}
    \label{tab:practice}
\end{table}

\subsection{Related work}

\paragraph{Optimisation theory} First-order optimisers leverage the first-order Taylor expansion of the objective function $\el(\vw)$---in particular, the gradient $\nabla_\vw\el(\vw)$. Theoretical treatments include mirror descent \citep{nemirovsky_yudin_1983}, 
natural gradient descent \citep{amari} and the Gauss-Newton method \citep{gauss-newton}. These methods have been explored in the context of deep learning \citep{revisiting-ngd,azizan2018stochastic,sun2022mirror}. First-order methods are amenable to deep learning since the gradient of the objective is available via recursive application of the chain rule---a.k.a.\ error back-propagation \citep{Rumelhart1986LearningRB}.

Second-order optimisers leverage the second-order Taylor expansion of the objective function $\el(\vw)$---in particular, the gradient $\nabla_\vw\el(\vw)$ and Hessian $\nabla^2_\vw\el(\vw)$. Examples include Newton's method \citep{Nocedal1999NumericalO} and cubic-regularised Newton's method \citep{Nesterov2006CubicRO}. Naïvely, second-order methods are less amenable to deep learning since the cost of the relevant Hessian computations is prohibitive at high dimension. That being said, efforts have been made to circumvent this issue \citep{hessian-linear}.

The majorise-minimise meta-algorithm \citep{mm} is an algorithmic pattern that can be used to derive optimisers. To apply the meta-algorithm, one must first derive an upper bound on the objective which matches the objective up to $k$th-order in its Taylor series for some integer $k$. This \textit{majorisation} can then be minimised as a proxy for reducing the original objective. \cref{fig:maj-min} illustrates the meta-algorithm for $k=1$.

\paragraph{Deep learning theory} The \textit{Lipschitz smoothness assumption}---a global constraint on the eigenvalues of the Hessian---is often used to derive and analyse neural network optimisers \citep{Agarwal2016FindingAL}. But this assumption has been questioned \citep{Zhang2020Why} and evidence has even been found for the reverse relationship, where the Hessian spectrum is highly sensitive to the choice of optimiser \citep{cohen2021gradient}.

These considerations motivate the development of theory that is more explicitly tailored to neural architecture. For instance, \citet{my-fromage} used an architectural perturbation bound termed \textit{deep relative trust} to characterise the neural network optimisation landscape as a function of network depth. Similarly, \citet{Yang2021TensorPI} sought to understand the role of width, leading to their \textit{maximal update parameterisation}. \cref{tab:practice,tab:theory} provide some points of comparison between automatic gradient descent and these and other frameworks.

\begin{figure}
\begin{minipage}{\textwidth}
\centering
\raisebox{-0.5\height}{\includegraphics{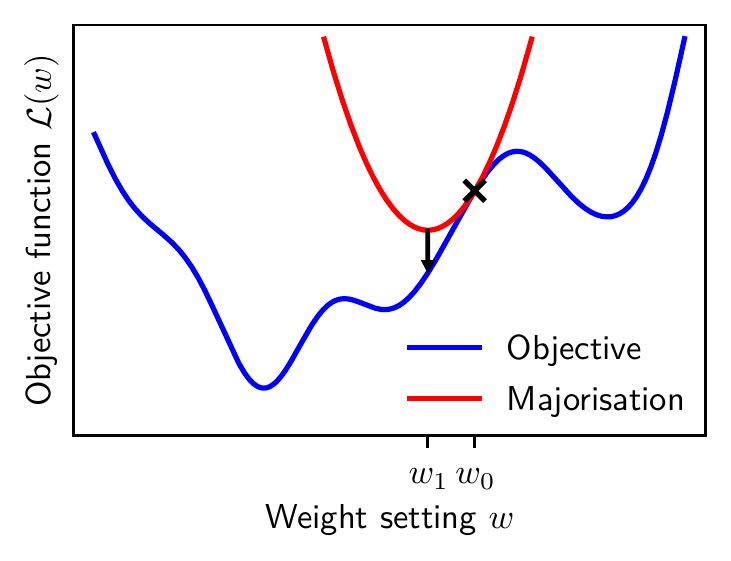}} \hfill     \raisebox{-0.5\height}{
        \begin{tikzpicture}
        [thick, block/.style={draw, minimum size=1.1cm}]
        
        \node [block,fill=red!30]    (a)               {$\el$};
        \node [block,fill=green!30]  (b) [right =of a] {$\ell$};
        \node [block,fill=orange!30] (c) [right =of b] {$\vf$};
        \node [block,fill=blue!30]   (d) [right =of c] {$\vw$};
        
        \node [block,fill=red!10]    (u) [below =of a] {$\Delta \el$};
        \node [block,fill=green!10]  (v) [below =of b] {$\Delta \ell$};
        \node [block,fill=orange!10] (w) [below =of c] {$\Delta \vf$};
        \node [block,fill=blue!10]   (x) [above =of d] {$\Delta \vw$};
        
        \node (p) [left  =of x, align=right, xshift=0.8cm, yshift=-0.04cm] {perturbation applied\\by optimiser};
        \node (q) [below =of v, yshift=0.9cm, align=center] {$\underbrace{\hspace{14.5em}}$\\[0.5ex]perturbations induced by optimiser};
        \node [below =of d, yshift=0.9cm, align=center] {weights};
        \node [above =of c, yshift=-0.825cm, align=center] {model};
        \node [above =of b, yshift=-0.825cm, align=center] {loss};
        \node [above =of a, yshift=-0.89cm, align=center] {objective};
        
        \draw[-latex] (b) edge (a);
        \draw[-latex] (c) edge (b);
        \draw[-latex] (d) edge (c);
        
        \draw[-latex] (a) edge (u);
        \draw[-latex] (b) edge (v);
        \draw[-latex] (c) edge (w);
        \draw[-latex] (x) edge (d);
        \end{tikzpicture}}
\end{minipage}
\caption{\captiontitle{Majorise-minimise and the perturbation hierarchy.} The \captiontitle{left panel} depicts the majorise-minimise meta-algorithm \citep{mm}, which is an algorithmic pattern for reducing an objective (blue) by minimising a sequence of upper bounds (one shown in red). The upper bounds, known as a \textit{majorisation}, must lie tangent to the objective to guarantee an improvement in one step of the meta-algorithm. The \captiontitle{right panel} depicts the perturbation hierarchy of a generic machine learning model: the optimiser perturbs the weights and this induces perturbations to the model output, the loss on individual training examples and ultimately the overall objective. Majorising machine learning objective functions requires addressing the full perturbation hierarchy.}
\label{fig:maj-min}
\end{figure}
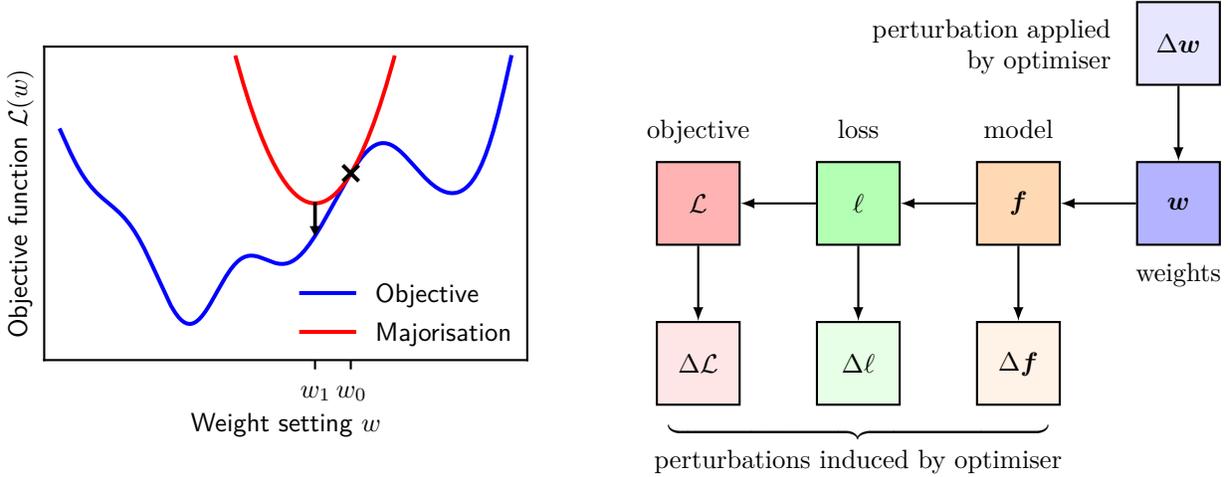

\subsection{Preliminaries}

Given a vector $\vv$ in $\R^n$, we will need to measure its size in three different ways:

\begin{definition}[Manhattan norm] The \textit{Manhattan norm} $\norm{\,\cdot\,}_1$ of a vector $\vv$ is defined by $\norm{\vv}_1 \defeq \sum_{i} \abs{\vv_i}$.
\end{definition}

\begin{definition}[Euclidean norm] The \textit{Euclidean norm} $\norm{\,\cdot\,}_2$ of a vector $\vv$ is defined by $\smash{\norm{\vv}_2 \defeq \sqrt{\sum_{i} \vv_i^2}}$.
\end{definition}

\begin{definition}[Infinity norm] The \textit{infinity norm} $\norm{\,\cdot\,}_\infty$ of a vector $\vv$ is defined by $\norm{\vv}_\infty \defeq \max_{i} \abs{\vv_i}$.
\end{definition}

For a matrix $\mM$ in $\R^{m \times n}$, the reader should be aware that it has a singular value decomposition:
\begin{fact}[SVD] Every matrix $\mM$ in $\R^{m\times n}$ admits a \textit{singular value decomposition} (SVD) of the form $\mM = \sum_{i=1}^{\min(m,n)} \sigma_i(\mM) \cdot \vu_i \vv_i^\top$ where the \textit{left singular vectors} $\{\vu_i\}$ are orthonormal vectors in $\R^{m}$, the \textit{right singular vectors} $\{\vv_i\}$ are orthonormal vectors in $\R^{m}$ and the \textit{singular values} $\{\sigma_i(\mM)\}$ are non-negative scalars.
\end{fact}

The singular value decomposition allows us to measure the size of a matrix in two different ways:

\begin{definition}[Frobenius norm] The \textit{Frobenius norm} $\norm{\,\cdot\,}_F$ of a matrix $\mM$ is given by $\norm{\mM}_F \defeq \sqrt{\sum_{i} \sigma_i(\mM)^2}$.
\end{definition}
\begin{definition}[Operator norm] The \textit{operator norm} $\norm{\,\cdot\,}_*$ of a matrix $\mM$ is given by $\norm{\mM}_* \defeq \max_i \sigma_i(\mM)$.
\end{definition}
While the operator norm $\norm{\mM}_*$ reports the largest singular value, the quantity $\norm{\mM}_F / \sqrt{\min(m,n)}$ reports the root mean square singular value. Finally, we will need to understand two aspects of matrix conditioning:
\begin{definition}[Rank] The \textit{rank} of a matrix counts the number of non-zero singular values.
\end{definition}
\begin{definition}[Stable rank]
The \textit{stable rank} of a matrix $\mM$ is defined by $\srank \mM \defeq \norm{\mM}_F^2 / \norm{\mM}_*^2$.
\end{definition}
The stable rank provides an approximation to the rank that ignores the presence of very small singular values. Let us consider the extremes. An orthogonal matrix $\mO\in\R^{m\times n}$ has both full rank and full stable rank: $\rank \mO = \srank \mO = \min(m,n)$. A rank-one matrix $\mP$ has unit stable rank and satisfies $\norm{\mP}_* = \norm{\mP}_F$.
\section{Majorise-Minimise for Generic Learning Problems}
\label{sec:mm-ml}

\begin{table}
    \centering
    \begin{tabularx}{\textwidth}{Xlcc}
        \toprule
            \textbf{Theory} & \textbf{Reference} & \makecell{\textbf{Handles the Loss}\\\begin{tikzpicture}[thick, block/.style={draw, minimum size=0.6cm}]
    \node [block,fill=red!30]    (a)               {$\el$};
    \node [block,fill=green!30]  (b) [right=0.5cm of a] {$\ell$};
    \node [block,fill=orange!30] (c) [right=0.5cm of b] {$\vf$};
    \draw[-latex] (b) edge (a);
    \draw[-latex] (c) edge (b);
\end{tikzpicture}} & \makecell{\textbf{Non-Linear Network}\\ \begin{tikzpicture}
        [thick, block/.style={draw, minimum size=0.6cm}]
        \node [block,fill=orange!30] (c) {$\vf$};
        \node [block,fill=blue!30]   (d) [right= 0.5cm of c] {$\vw$};
        \draw[-latex] (d) edge (c);
        \end{tikzpicture}} \\
        \midrule
        mirror descent & $\mathrlap{\text{\citet{nemirovsky_yudin_1983}}}$\hspace{10em} & \cmark  & \xmark \\
        Gauss-Newton method &\citet{gauss-newton}& \cmark  & \xmark \\
        natural gradient descent &\citet{amari}& \cmark  & \xmark \\
        neural tangent kernel &\citet{NTKjacot}& \cmark  & \xmark \\
        deep relative trust &\citet{my-fromage}& \xmark  & \cmark \\
        tensor programs & \citet{Yang2021TensorPI}& \xmark  & \cmark \\
        automatic gradient descent & this paper &\cmark  & \cmark \\
        \bottomrule
\end{tabularx}
    \caption{\captiontitle{Comparing popular frameworks for first-order optimisation theory.} Frameworks differ in whether they can handle the interaction between the model output $\vf$ and the objective $\el$, and the complex non-linear interaction between the weights $\vw$ and the model output $\vf$. Our framework handles both aspects.}
    \label{tab:theory}
\end{table}

This section develops a framework for applying the majorise-minimise meta-algorithm to generic optimisation problems in machine learning. In particular, the novel technique of \textit{functional expansion} is introduced. \cref{sec:mm-dnn} will apply this technique to deep neural networks. All proofs are supplied in \cref{app:proofs}.

Given a machine learning model and a set of training data, our objective is to minimise the error of the model, averaged over the training data. Formally, we would like to minimise the following function:

\begin{definition}[Composite objective] Consider a machine learning model $\vf$ that maps an input $\vx$ and a weight vector $\vw$ to output $\vf(\vx;\vw)$. Given data $\set{S}$ and a convex loss function $\ell$, the \textit{objective} $\el(\vw)$ is defined by:
\begin{equation*}
    \el(\vw) \defeq \frac{1}{|\set{S}|}\sum_{(\vx,\vy) \in \set{S}} \ell(\vf(\vx;\vw), \vy).
\end{equation*}
\end{definition}
We refer to this objective as \textit{composite} since the loss function $\ell$ is \textit{composed} with a machine learning model $\vf$. While the loss function itself is convex, the overall composite is often non-convex due to the non-linear machine learning model. Common convex loss functions include the square loss and the cross-entropy loss:

\begin{example}[Square loss]\label{ex:sq-loss} The \textit{square loss} is defined by: $\ell(\vf(\vx; \vw), \vy) \defeq \frac{1}{2d_L} \norm{\vf(\vx; \vw) - \vy}_2^2$.
\end{example}
\begin{example}[Xent loss]\label{ex:xent-loss} The \textit{cross-entropy (xent) loss} is defined by: $\ell(\vf(\vx), \vy) \defeq - \log [\softmax(\vf(\vx))]^\top \vy$, where the softmax function is defined by $\softmax(\vf(\vx))\defeq \exp \vf(\vx) / \norm{\exp \vf(\vx)}_1$.
\end{example}

\subsection{Decomposition of linearisation error}

First-order optimisers leverage the linearisation of the objective at the current iterate. To design such methods, we must understand the realm of validity of this linearisation. To that end, we derive a very general decomposition of the linearisation error of a machine learning system. The result is stated in terms of a \textit{perturbation hierarchy}. In particular, perturbing the weight vector of a machine learning model $\vw \to \vw + \Delta \vw$ induces perturbations to the model output $\vf \to \vf + \Delta \vf$, to the loss on individual data samples $\ell \to \ell + \Delta \ell$ and, at last, to the overall objective function $\el \to \el + \Delta \el$. Formally, a weight perturbation $\Delta \vw$ induces:
\begin{flalign*}
    &\Delta \vf(\vx) &&\coloneqq \vf(\vx;\vw+\Delta \vw) - \vf(\vx; \vw); \hspace{16em} \tag{functional perturbation}\\
    &\Delta \ell(\vf(\vx), \vy) &&\coloneqq  \ell(\vf(\vx)+\Delta \vf(\vx),\vy) - \ell(\vf(\vx),\vy); \tag{loss perturbation}\\
    &\Delta \el(\vw) &&\coloneqq \textstyle\frac{1}{|\set{S}|}\sum_{(\vx,\vy) \in \set{S}} \Delta \ell(\vf(\vx), \vy) \tag{objective perturbation}.
\end{flalign*}
We have adopted a compact notation where the dependence of $\vf(\vx;\vw)$ on $\vw$ is at times suppressed. The perturbation hierarchies of a generic machine learning model and a deep neural network are visualised in \cref{fig:maj-min,fig:apbs}, respectively. The linearisation error of the objective perturbation $\Delta \el$ decomposes as:

\begin{restatable}[Decomposition of linearisation error]{proposition}{decomposition}\label{thm:decomposition}For any differentiable loss $\ell$ and any differentiable machine learning model $\vf$ the linearisation error of the objective function $\el$ admits the following decomposition:
    \begin{align*}
    \quad\quad\quad\underbrace{\Delta \el(\vw) - \nabla_\vw\el(\vw)^\top \Delta \vw}_{\mathclap{\text{linearisation error of objective}}} \quad\quad&= &&\frac{1}{|\set{S}|}\sum_{(\vx,\vy)\in \set{S}} \nabla_{\vf(\vx)} \ell(\vf(\vx),\vy)^\top \underbrace{\left[\Delta \vf(\vx) - \nabla_\vw \vf(\vx) \Delta \vw \right]}_{\mathclap{\text{linearisation error of model}}} \\ &&+\,&\frac{1}{|\set{S}|}\sum_{(\vx,\vy)\in \set{S}}\underbrace{\Delta \ell(\vf(\vx), \vy) -\nabla_{\vf(\vx)}\ell(\vf(\vx),\vy)^\top \Delta \vf(\vx)}_{\text{linearisation error of loss}}.\quad\quad\quad \nonumber
    \end{align*}
\end{restatable}
In words: the linearisation error of the objective decomposes into two terms. The first depends on the linearisation error of the machine learning model and the second the loss. This decomposition relies on nothing but differentiability. For a convex loss, the second term may be interpreted as a Bregman divergence:

\begin{definition}[Bregman divergence of loss]\label{def:bregman} For any convex loss $\ell$:
\begin{flalign*}
    \qquad\qquad\qquad\qquad\bregman_{\ell(\cdot,\vy)}(\vf(\vx), \Delta \vf(\vx)) \defeq \Delta \ell(\vf(\vx), \vy) -\nabla_{\vf(\vx)}\ell(\vf(\vx),\vy)^\top \Delta \vf(\vx). &&
\end{flalign*}
\end{definition}

 A Bregman divergence is just the linearisation error of a convex function. Two important examples are:

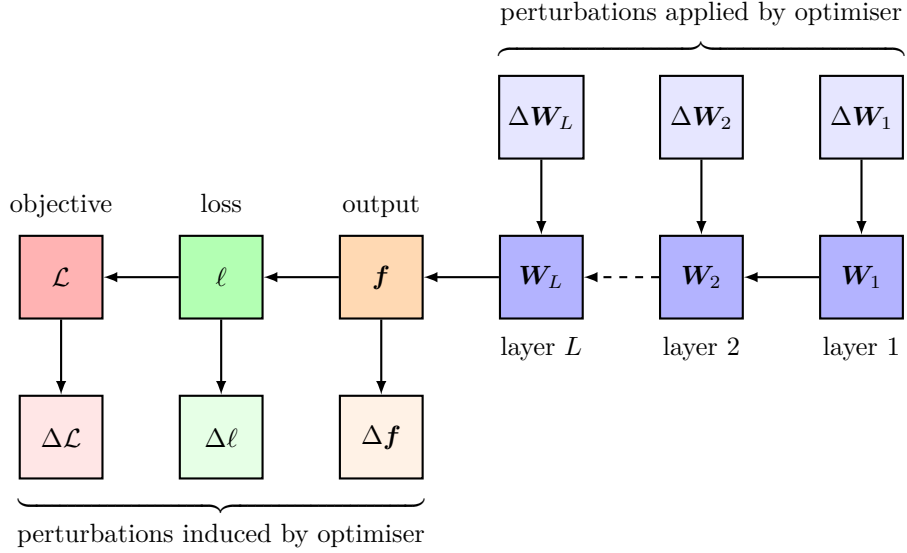
\begin{figure}
    \centering
    \begin{tikzpicture}
        [thick, block/.style={draw, minimum size=1.1cm}]
        
        \node [block,fill=red!30]    (aa)               {$\el$};
        \node [block,fill=green!30]  (bb) [right =of aa] {$\ell$};
        
        \node [block,fill=orange!30]    (a) [right =of bb]               {$\vf$};
        \node [block,fill=blue!30]      (b) [right =of a] {$\mW_L$};
        \node [block,fill=blue!30]      (c) [right =of b] {$\mW_2$};
        \node [block,fill=blue!30]      (d) [right =of c] {$\mW_1$};
        
        \node [block,fill=orange!10] (v) [below =of a] {$\Delta \vf$};
        \node [block,fill=blue!10]   (x) [above =of b] {$\Delta \mW_L$};
        \node [block,fill=blue!10]   (y) [above =of c] {$\Delta \mW_2$};
        \node [block,fill=blue!10]   (z) [above =of d] {$\Delta \mW_1$};
        
        \node [block,fill=red!10]    (uu) [below =of aa] {$\Delta \el$};
        \node [block,fill=green!10]  (vv) [below =of bb] {$\Delta \ell$};
        
        \node (q) [above =of y, yshift=-0.9cm, align=center] {perturbations applied by optimiser\\[0.5ex]$\overbrace{\hspace{15.5em}}$};
        \node (p) [below =of vv, yshift=0.9cm, align=center] {$\underbrace{\hspace{15.5em}}$\\[0.5ex]perturbations induced by optimiser};
        
        \node [below =of d, yshift=0.9cm, align=center] {layer $1$};
        \node [below =of c, yshift=0.9cm, align=center] {layer $2$};
        \node [below =of b, yshift=0.9cm, align=center] {layer $L$};
        \node [above =of a, yshift=-0.89cm, align=center] {output};
        \node [above =of bb, yshift=-0.825cm, align=center] {loss};
        \node [above =of aa, yshift=-0.89cm, align=center] {objective};
        
        \draw[-latex]           (bb) edge (aa);
        \draw[-latex]           (a) edge (bb);
        \draw[-latex]           (b) edge (a);
        \draw[-latex, dashed]   (c) edge (b);
        \draw[-latex]           (d) edge (c);
        
        \draw[-latex] (a) edge (v);
        \draw[-latex] (x) edge (b);
        \draw[-latex] (y) edge (c);
        \draw[-latex] (z) edge (d);
        
        \draw[-latex] (aa) edge (uu);
        \draw[-latex] (bb) edge (vv);
        
        \end{tikzpicture}
    \caption{\captiontitle{Perturbation hierarchy of a deep neural network.} When training a neural network, the optimiser applies structured perturbations to the weights, in the form of one perturbation matrix $\Delta \mW_k$ per weight matrix $\mW_k$. Deep relative trust \citep{my-fromage} provides a tool to understand how structured weight perturbations of this form affect the network output $\vf$. Combining deep relative trust with a Bregman divergence \citep{bregman1967relaxation} allows us to analyse the full perturbation hierarchy.}
    \label{fig:apbs}
\end{figure}

\begin{restatable}[Bregman divergence of square loss]{lemma}{squarebreg}\label{lem:sq-bregman}
When $\ell$ is set to square loss, then:
\begin{flalign*}
    \qquad\qquad\qquad\qquad\bregman_{\ell(\cdot,\vy)}(\vf(\vx), \Delta \vf(\vx)) = \tfrac{1}{2d_L} \norm{\Delta \vf(\vx)}_2^2.&&
\end{flalign*}
\end{restatable}

\begin{restatable}[Bregman divergence of xent loss]{lemma}{xentbreg} \label{lem:xent-bregman}
When $\ell$ is set to cross-entropy loss, and if $\vy^\top \bm{1} =1$, then:
    \begin{flalign*}
        \qquad\qquad\qquad\qquad\bregman_{\ell(\cdot,\vy)}(\vf(\vx), \Delta \vf(\vx)) &= \kl \Big(\softmax(\vf(\vx))\,\Big|\Big|\, \softmax(\vf(\vx)+\Delta \vf(\vx))\Big)&& \\
        &\leq \half\norm{\Delta \vf(\vx)}_\infty^2 + \mathcal{O}(\Delta \vf^3).&&
    \end{flalign*}
\end{restatable}

Our methods may be applied to other convex losses by calculating or bounding their Bregman divergence.

\subsection{Functional expansion and functional majorisation}

Before continuing, we make one simplifying assumption. Observe that the first term on the right-hand side of \cref{thm:decomposition} is a high-dimensional inner product between two vectors. Since there is no clear reason why these two vectors should be aligned, let us assume that their inner product is zero:
\begin{assumption}[Orthogonality of model linearisation error]\label{ass:orthog}
In the same setting as \cref{thm:decomposition}:
\begin{equation*}
    \frac{1}{|\set{S}|}\sum_{(\vx,\vy)\in \set{S}} \nabla_{\vf(\vx)} \ell(\vf(\vx),\vy)^\top \underbrace{\left[\Delta \vf(\vx) - \nabla_\vw \vf(\vx) \Delta \vw \right]}_{\mathclap{\text{linearisation error of model}}} = 0.
\end{equation*}
\end{assumption}

While it is possible to work without this assumption \citep{bernstein-thesis}, we found that its inclusion simplifies the analysis and in practice did not lead to a discernible weakening of the resulting algorithm. In any case, this assumption is considerably milder than the common assumption in the literature \citep{revisiting-ngd,NEURIPS2019_0d1a9651} that the model linearisation error is itself zero: $\left[\Delta \vf(\vx) - \nabla_\vw \vf(\vx) \Delta \vw \right] = 0$.

Armed with \cref{thm:decomposition} and \cref{ass:orthog}, we are ready to introduce functional expansion and majorisation:

\begin{restatable}[Functional expansion]{theorem}{functmajor}\label{thm:functmajor}Consider a convex differentiable loss $\ell$ and a differentiable machine learning model $\vf$. Under \cref{ass:orthog}, the corresponding composite objective $\el$ admits the expansion:
    \begin{align*}
    \el(\vw + \Delta \vw) = \underbrace{\el(\vw) + \nabla_\vw\el(\vw)^\top \Delta \vw}_{\text{first-order Taylor series}} +\frac{1}{|\set{S}|}\sum_{(\vx,\vy)\in \set{S}}\bregman_{\ell(\cdot,\vy)}(\vf(\vx), \Delta \vf(\vx)).
    \end{align*}
\end{restatable}
So the perturbed objective $\el(\vw+\Delta \vw)$ may be written as the sum of its first-order Taylor expansion with a Bregman divergence in the model outputs averaged over the training set.
It is straightforward to specialise this result to different losses by substituting in their Bregman divergence:

\begin{restatable}[Functional expansion of mean squared error]{corollary}{sqmajor}\label{lem:sq-major} Under \cref{ass:orthog}, for square loss:
    \begin{flalign*}
    \qquad\qquad\qquad\qquad\el(\vw + \Delta \vw) = \el(\vw) + \nabla_\vw\el(\vw)^\top \Delta \vw +\frac{1}{|\set{S}|}\sum_{(\vx,\vy)\in \set{S}}\tfrac{1}{2d_L} \norm{\Delta \vf(\vx)}_2^2.&&
    \end{flalign*}
\end{restatable}

\begin{restatable}[Functional majorisation for xent loss]{corollary}{xentmajor}\label{lem:xent-major}
Under \cref{ass:orthog}, for cross-entropy loss, if $\vy^\top \bm{1} =1$:
    \begin{flalign*}
    \qquad\qquad\qquad\qquad\el(\vw + \Delta \vw) \leq \el(\vw) + \nabla_\vw\el(\vw)^\top \Delta \vw +\frac{1}{|\set{S}|}\sum_{(\vx,\vy)\in \set{S}}\half\norm{\Delta \vf(\vx)}_\infty^2 + \mathcal{O}(\Delta \vf^3).&&
    \end{flalign*}
\end{restatable}

When the functional perturbation is reasonably ``spread out'', we would expect $\norm{\Delta \vf(\vx)}_\infty^2 \approx \norm{\Delta \vf(\vx)}_2^2/d_L$. In this setting, the functional majorisation of cross-entropy loss agrees with the functional expansion of mean squared error to second order. While the paper derives automatic gradient descent for the square loss, this observation justifies its application to cross-entropy loss, as in the case of the ImageNet experiments.

\subsection{Recovering existing frameworks}
\label{sec:recover}

We briefly observe that three existing optimisation frameworks may be recovered efficiently from \cref{thm:functmajor}:

\paragraph{Mirror descent} For linear models $\vf(\vx;\mW) \defeq \mW \vx$, the Bregman divergence $\bregman_{\ell(\cdot,\vy)}(\vf(\vx), \Delta \vf(\vx))$ may be written $\bregman_{\ell(\cdot,\vy)}(\mW\vx, \Delta\mW\vx)$. This is a convex function of the weight perturbation $\Delta \mW$. Substituting into \cref{thm:functmajor} and minimising with respect to $\Delta \mW$ is the starting point for mirror descent.

\paragraph{Gauss-Newton method} Substituting the linearised functional perturbation $\Delta \vf(\vx) \approx \nabla_\vw \vf(\vx) \Delta \vw$ into \cref{lem:sq-major} and minimising with respect to $\Delta \vw$ is the starting point for the Gauss-Newton method.

\paragraph{Natural gradient descent} Substituting the linearised functional perturbation $\Delta \vf(\vx) \approx \nabla_\vw \vf(\vx) \Delta \vw$ into \cref{lem:xent-major} and minimising with respect to $\Delta \vw$ is the starting point for natural gradient descent.
\section{Majorise-Minimise for Deep Learning Problems}
\label{sec:mm-dnn}

In this section, we will focus our efforts on deriving an optimiser for deep fully-connected networks trained with square loss. The derivation for cross-entropy loss is analogous. Proofs are relegated to \cref{app:proofs}. 

\begin{definition}[Fully-connected network]\label{def:dln}
A \textit{fully-connected network (FCN)} $\vf$ of depth $L$ maps an input $\vx\in\R^{d_0}$ to an output $\vf(\vx;\vw) \in \R^{d_L}$ via $L$ matrix multiplications interspersed by non-linearity $\relu(z) \defeq \max(0,z)$:
\begin{equation*}
\vf(\vx; \vw) \coloneqq \mW_L\circ(\relu{}\circ \mW_{L - 1}) \circ(\relu{}\circ \mW_{L - 2}) \circ \dots  \circ (\relu{} \circ \mW_1 \vx).
\end{equation*}
\end{definition}

In this expression, $\vw$ denotes the tuple of matrices $\vw = (\mW_1,...,\mW_L)$ with $k$th matrix $\mW_k$ in $\R^{d_k\times d_{k-1}}$. In what follows, we will find the following dimensional scaling to be particularly convenient:
\begin{prescription}[Dimensional scaling]\label{prescription:norm} For $\eta>0$, the data $(\vx,\vy)$, weights $\mW_k$ and updates $\Delta\mW_k$ should obey:
\begin{align*}
    \norm{\vx}_2 &= \sqrt{d_0}; \tag{input scaling} \\
    \norm{\mW_k}_* &= \sqrt{d_k/d_{k-1}} \hspace{1.519em}\qquad\text{for all }k=1,...,L; \tag{weight scaling} \\
    \norm{\Delta \mW_k}_* &= \sqrt{d_k/d_{k-1}} \cdot \tfrac{\eta}{L} \qquad\text{for all }k=1,...,L; \tag{update scaling}\\
    \norm{\vy}_2 &= \sqrt{d_L}. \tag{target scaling}
\end{align*}
\end{prescription}
While results can be derived without adopting \cref{prescription:norm}, the scalings substantially simplify our formulae. One reason for this is that, under \cref{prescription:norm}, we have the telescoping property that $\prod_{k=1}^L \norm{\mW_k}_* = \sqrt{d_L/d_0}$. For a concrete example of how this helps, consider the following bound on the norm of the network outputs:

\begin{restatable}[Output bound]{lemma}{outbound}
\label{lem:outbound} The output norm of a fully-connected network $\vf$ obeys the following bound:
\begin{align*}
    \norm{\vf(\vx;\vw)}_2 &\leq \left[\prod_{k=1}^L \norm{\mW_k}_* \right] \times \norm{\vx}_2 = \sqrt{d_L} \text{ under \cref{prescription:norm}}.
\end{align*}
\end{restatable}

So, under \cref{prescription:norm}, the bound is simple. Furthermore, the scaling of the update with a single parameter $\eta$ reduces the problem of solving for an optimiser to a single parameter problem. To see how this might make life easier, consider the following lemma that relates weight perturbations to functional perturbations:

\begin{restatable}[Deep relative trust]{lemma}{archbounds}
\label{lem:deep_perturbation_bounds}
When adjusting the weights $\vw = (\mW_1,...,\mW_L)$ of a fully-connected network $\vf$ by $\Delta\vw = (\Delta\mW_1,...,\Delta\mW_L)$, the induced functional perturbation $\Delta \vf(\vx)\defeq\vf(\vx;\vw+\Delta\vw)-\vf(\vx;\vw)$ obeys:
\begin{align*}
    \norm{\Delta\vf(\vx)}_2 &\leq \left[\prod_{k=1}^L \norm{\mW_k}_* \right] \times \norm{\vx}_2 \times \left[ \prod_{k = 1}^L \left( 1 + \frac{\Vert \Delta \mW_k \Vert_{*}}{\Vert \mW_k \Vert_{*}}\right)  - 1 \right] \leq \sqrt{d_L}\times(\exp \eta - 1) \text{ under \cref{prescription:norm}}.
\end{align*}
\end{restatable}
So, under \cref{prescription:norm}, the single parameter $\eta$ directly controls the size of functional perturbations.

In terms of enforcing \cref{prescription:norm} in practice, the norms of the data $(\vx,\vy)$ may be set via pre-processing, the norm of the update $\Delta \mW_k$ may be set via the optimisation algorithm and the norm of the weight matrix $\mW_k$ may be set by the choice of initialisation. While, yes, $\norm{\mW_k}_*$ may drift during training, the amount that this can happen is limited by \citet{Weyl1912}'s inequality for singular values. In particular, after one step the perturbed operator norm $\norm{\mW_k + \Delta \mW_K}_*$ is sandwiched like $(1-\eta/L) \cdot \norm{\mW_k}_* \leq \norm{\mW_k + \Delta \mW_K}_* \leq (1+\eta/L) \cdot\norm{\mW_k}_*$.

\begin{algorithm}[t]
\caption{\captiontitle{Automatic gradient descent.} The matrix $\mW_k$ in $\R^{d_k \times d_{k-1}}$ is the weight matrix at layer $k$. The gradient $\nabla_{\mW_k} \el$ is with respect to the objective $\el$ evaluated on a mini-batch $B$ of training samples.}\label{alg:agd}
\begin{algorithmic}
\tt
\setstretch{1.8}\vspace{0.5em}
\DEF[initialise\_weights]
\FOR{layer $k$ in $\{1,...,L\}$:}
\STATE $\mW_k \sim \uniform(\mathtt{orthogonal}(d_k,d_{k-1}))$ \WCOMMENT{sample a semi-orthogonal matrix}
\STATE $\mW_k \gets \mW_k \cdot \sqrt{\frac{d_k}{d_{k-1}}}$ \WCOMMENT{rescale its singular values}
\ENDFOR
\ENDDEF
    \vspace{-1.6ex}\DEF[update\_weights]
    \STATE $G \gets \frac{1}{L}\sum_{l=1}^L \norm{\nabla_{\mW_k} \el}_F \cdot \sqrt{\frac{d_k}{d_{k-1}}}$ \WCOMMENT{get gradient summary}
\STATE $\smash{\eta \gets \log\frac{1 + \sqrt{1+ 4G }}{2}}$ \WCOMMENT{set automatic learning rate}
\FOR{layer $k$ in $\{1,...,L\}$:}
    \STATE $\mW_k \gets \mW_k - \frac{\eta}{L} \cdot \frac{\nabla_{\mW_k} \el}{\norm{\nabla_{\mW_k} \el}_F} \cdot \sqrt{\frac{d_k}{d_{k-1}}}$ \WCOMMENT{update weights}
\ENDFOR
\ENDDEF
\setstretch{1.0}
\end{algorithmic}
\end{algorithm}

\subsection{Deriving automatic gradient descent}

With both functional majorisation and deep relative trust in hand, we can majorise the deep network objective:

\begin{restatable}[Exponential majorisation]{lemma}{majordnn}\label{lem:sq-major-nn}
For an FCN with square loss, under \cref{ass:orthog} and \cref{prescription:norm}:
    \begin{equation*}
        \el(\vw+\Delta \vw) \leq \el(\vw) + \frac{\eta}{L}\sum_{k=1}^L\left[\sqrt{d_k/d_{k-1}} \times\trace\frac{\Delta \mW_k^\top\nabla_{\mW_k}\el}{\norm{\Delta \mW_k}_*}\right] + \tfrac{1}{2} \,(\exp \eta -1)^2.
    \end{equation*}
\end{restatable}

Observe that the majorisation only depends on the magnitude of the scalar $\eta$ and on some notion of angle $\trace\Delta \mW_k^\top\nabla_{\mW_k}\el/\norm{\Delta \mW_k}_*$ between the perturbation matrix $\Delta \mW_k$ and the gradient matrix $\nabla_{\mW_k}\el$. To derive an optimiser, we would now like to minimise this majorisation with respect to $\eta$ and this angle. First, let us introduce one additional assumption and one additional definition:
\begin{assumption}[Gradient conditioning]\label{approx:g-cond} The gradient satisfies $\srank\nabla_{\mW_k}\el=1$ at all layers $k=1,...,L$.
\end{assumption}
This assumption implies that the Frobenius norm $\norm{\nabla_{\mW_k}\el}_F$ and operator norm $\norm{\nabla_{\mW_k}\el}_*$ of the gradient at layer $k$ are equal. It is not immediately obvious why this should be a good assumption. After all, the gradient is a sum of $\abs{\set{S}}$ rank-one matrices: $\nabla_{\mW_k}\el = \tfrac{1}{\abs{\set{S}}} \sum_{(\vx,\vy)\in\set{S}} \nabla_{\vh_k}\ell(\vf(\vx),\vy) \otimes \vh_{k-1}$, where $\vh_{k-1}(\vx)$ and $\vh_k(\vx)$ denote the inputs and outputs of the weight matrix $\mW_k$ at layer $k$, and $\otimes$ denotes the outer product. So, naïvely, one might expect the gradient $\nabla_{\mW_k}\el$ to have a stable rank of $\min(d_k,d_{k-1},\abs{\set{S}})$. But it turns out to be a good assumption in practice \citep{Yang2021TensorPI,yang2021tuning}. And for the definition:

\begin{definition}[Gradient summary]\label{def:gsummary}
At a weight setting $\vw$, the \textit{gradient summary} $G$ is given by:
\begin{align*}
        G & \defeq \frac{1}{L}\sum_{k=1}^L \sqrt{d_k/d_{k-1}} \cdot \norm{ \nabla_{\mW_k} \el(\vw)}_F.
\end{align*}
\end{definition}
The gradient summary is a weighted average of gradient norms over layers. It can be thought of as a way to measure the size of the gradient while accounting for the fact that the weight matrices at different layers may be on different scales. This is related to the concept of the \textit{gradient scale coefficient} of \citet{Philipp2017TheEG}.

We now have everything we need to derive automatic gradient descent via the majorise-minimise principle:

\begin{restatable}[Automatic gradient descent]{theorem}{loglr}\label{thm:log-lr}
For a deep fully-connected network, under \cref{ass:orthog,approx:g-cond} and \cref{prescription:norm}, the majorisation of square loss given in \cref{lem:sq-major-nn} is minimised by setting:
\begin{align*}
    \eta = \log\frac{1 + \sqrt{1+4G}}{2},\qquad
    \Delta \mW_k = - \frac{\eta}{L}\cdot \sqrt{d_k/d_{k-1}} \cdot\frac{\nabla_{\mW_k} \el}{\norm{\nabla_{\mW_k} \el}_F}, \qquad \text{for all layers } k=1,...,L.
\end{align*}
\end{restatable}

We present pseudocode for this theorem in \cref{alg:agd}, and a PyTorch implementation in \cref{app:pytorch}. Via a simple derivation based on clear algorithmic principles, automatic gradient descent unifies various heuristic and theoretical ideas that have appeared in the literature:
\begin{itemize}[leftmargin=*]
    \item \textit{Relative updates.} The update is scaled relative to the norm of the weight matrix to which it is applied---assuming the weight matrices are scaled according to \cref{prescription:norm}. Such a scaling was proposed by \citet{You:EECS-2017-156} and further explored by \citet{carbonnelle2019layer} and \citet{my-fromage}. There is evidence that such relative synaptic updates may occur in neuroscience \citep{Loewenstein9481}.
    \item \textit{Depth scaling.} Scaling the perturbation strength like $1/L$ for networks of depth $L$ was proposed on theoretical grounds by \citet{my-fromage} based on analysis via deep relative trust.
    \item \textit{Width scaling.} The dimensional factors of $d_k$ and $d_{k-1}$ that appear closely relate to the maximal update parameterisation of \citet{Yang2021TensorPI} designed to ensure hyperparameter transfer across network width.
    \item \textit{Gradient clipping.} The logarithmic dependence of the update on the gradient summary may be seen as an automatic form of \textit{adaptive gradient clipping} \citep{pmlr-v139-brock21a}---a technique which clips the gradient once its magnitude surpasses a certain threshold set by a hyperparameter.
\end{itemize}

\subsection{Convergence analysis}

This section presents theoretical convergence rates for automatic gradient descent. While the spirit of the analysis is standard in optimisation theory, the details may still prove interesting for their detailed characterisation of the optimisation properties of deep networks. For instance, we propose a novel Polyak-Łojasiewicz inequality tailored to the operator structure of deep networks. We begin with two observations:

\begin{restatable}[Bounded objective]{lemma}{objectivebound}\label{lem:objectivebound}
For square loss, the objective is bounded as follows:
\begin{align*}
    \el(\vw) &\leq \frac{1}{\abs{\set{S}}} \sum_{(\vx,\vy)\in \set{S}}\frac{\norm{\vf(\vx;\vw)}_2^2 +\norm{\vy}_2^2}{2d_L} \leq 1 \text{ under \cref{prescription:norm}.}
\end{align*}
\end{restatable}

\begin{restatable}[Bounded gradient]{lemma}{gradientbound}\label{lem:gradientbound}
For square loss, the norm of the gradient at layer $k$ is bounded as follows:
\begin{align*}
    \norm{\nabla_{\mW_k}\el}_F &\leq \frac{\prod_{l=1}^L\norm{\mW_l}_*}{\norm{\mW_k}_*} \cdot \sqrt{\frac{2\el(\vw)}{d_L}} \cdot \sqrt{\frac{1}{\abs{\set{S}}} \sum_{(\vx,\vy)\in \set{S}}\norm{\vx}_2^2} \leq \sqrt{2\cdot\frac{d_{k-1}}{d_k}} \text{ under \cref{prescription:norm}.}
\end{align*}
\end{restatable}

These results help us prove that automatic gradient descent converges to a point where the gradient vanishes:

\begin{restatable}[Convergence rate to critical point]{lemma}{criticalrate}\label{lem:criticalrate}
Consider a fully-connected network trained by automatic gradient descent (\cref{thm:log-lr}) and square loss for $T$ iterations. Let $G_t$ denote the gradient summary (\cref{def:gsummary}) at step $t\leq T$. Under \cref{ass:orthog,approx:g-cond} and \cref{prescription:norm}, AGD converges at the following rate:\vspace{-0.5em}
\begin{equation*}
    \min_{t\in\{1,...,T\}} G_t^2 \leq \frac{11}{T}.
\end{equation*}
\end{restatable}

This lemma can be converted into a convergence rate to a global minimum with one additional assumption:

\begin{assumption}[Deep Polyak-Łojasiewicz inequality] \label{ass:pl}
For some $\alpha>0$, the gradient norm is lower bounded by:
\begin{align*}
    \norm{\nabla_{\mW_k}\el}_F &\geq \alpha \times \frac{\prod_{l=1}^L\norm{\mW_l}_*}{\norm{\mW_k}_*} \cdot \sqrt{\frac{2\el(\vw)}{d_L}} \cdot \sqrt{\frac{1}{\abs{\set{S}}} \sum_{(\vx,\vy)\in \set{S}}\norm{\vx}_2^2} = \alpha \times \sqrt{2\cdot\el(\vw)\cdot\frac{d_{k-1}}{d_k}} \text{ under \cref{prescription:norm}.}
\end{align*}
\end{assumption}
This lower bound mirrors the structure of the upper bound in \cref{lem:gradientbound}. The parameter $\alpha$ captures how much of the gradient is attenuated by small singular values in the weights and by deactivated $\relu$ units. While Polyak-Łojasiewicz inequalities are common in the literature \citep{LIU202285}, our assumption is novel in that it pays attention to the operator structure of the network. \cref{ass:pl} leads to the following theorem:

\begin{restatable}[Convergence rate to global minima]{theorem}{globalrate}\label{thm:globalrate}
For automatic gradient descent (\cref{thm:log-lr}) in the same setting as \cref{lem:criticalrate} but with the addition of \cref{ass:pl}, the mean squared error objective at step $T$ obeys:
\begin{align*}
    \el(\vw_T) \leq \frac{1}{\alpha^2}\times\frac{6}{T}.
\end{align*}
\end{restatable}
\subsection{Experiments}

\begin{figure}
    \centering
    \includegraphics[width=\textwidth]{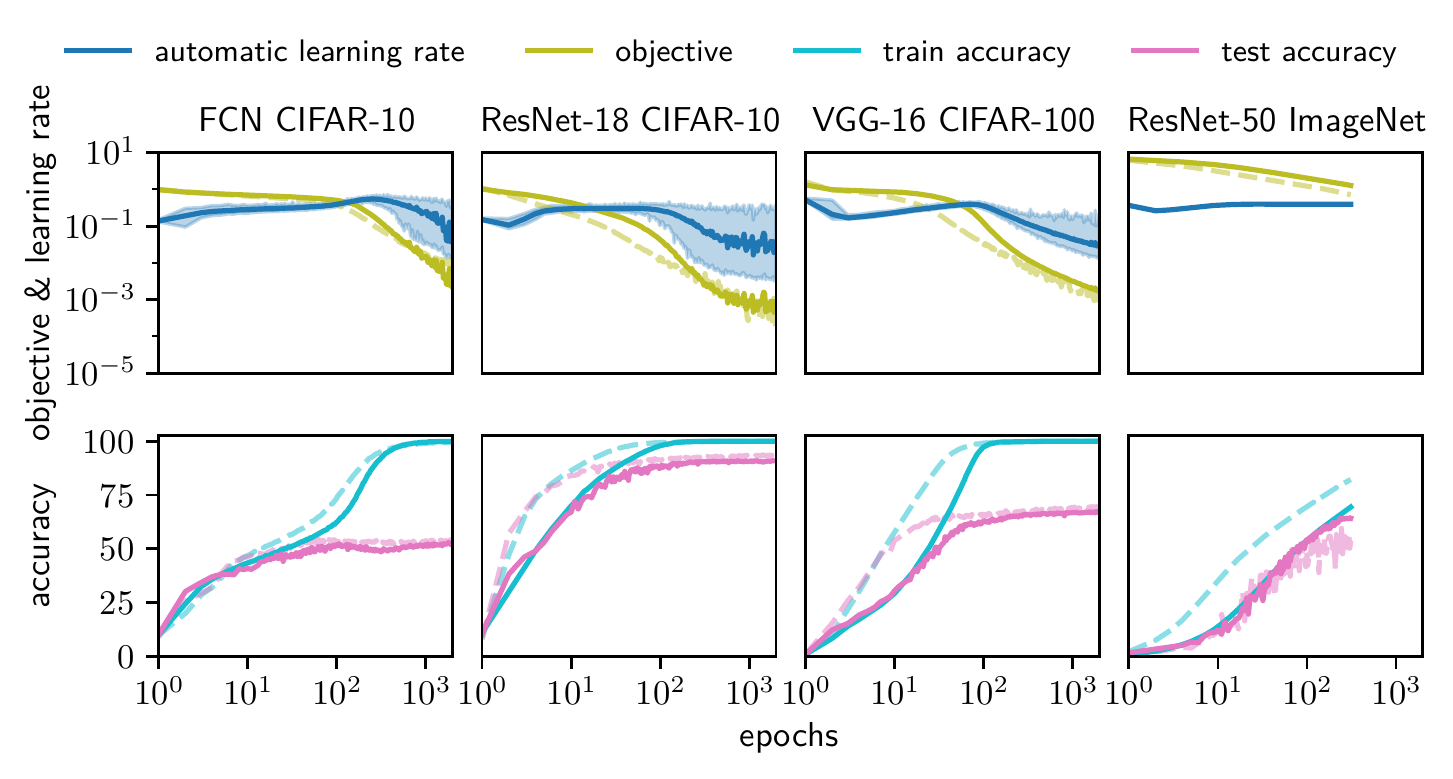}
    \caption{\captiontitle{Benchmarking automatic gradient descent on a range of architectures and datasets.} Solid lines are AGD and faint dashed lines are tuned Adam except for ImageNet where the dashed line is SGD with a fixed learning rate of 0.1. ImageNet used cross-entropy loss with a mini-batch size of 1024. The other experiments used square loss with a mini-batch size of 128.
    The \captiontitle{top row} plots the automatic learning rate ($\eta$ in the main text) and objective value. The maximum and minimum learning rate for each epoch is included in addition to the mean for the first three plots. The \captiontitle{bottom row} shows the train and test accuracy.
    }\label{fig:1}
\end{figure}

The goal of our experiments was twofold. First, we wanted to test automatic gradient descent (AGD, \cref{alg:agd}) on a broad variety of networks architectures and datasets to check that it actually works. In particular, we tested AGD on fully-connected networks (FCNs, \cref{def:dln}), and both VGG-style \citep{simonyan2015a} and ResNet-style \citep{He2015DeepRL} convolutional neural networks on the CIFAR-10, CIFAR-100 \citep{Krizhevsky09learningmultiple} and ImageNet \citep[ILSVRC2012]{deng2009imagenet} datasets with standard data augmentation. And second, to see what AGD may have to offer beyond the status quo, we wanted to compare AGD to tuned Adam and SGD baselines, as well as Adam and SGD run with their default hyperparameters.

To get AGD working with convolutional layers, we adopted a per-submatrix normalisation scheme. Specifically, for a convolutional tensor with filters of size $\mathtt{k_x} \times \mathtt{k_y}$, we implemented the normalisation separately for each of the $\mathtt{k_x} \times \mathtt{k_y}$ submatrices of dimension $\mathtt{channels_{in}} \times \mathtt{channels_{out}}$. Since AGD does not yet support biases or affine parameters in batchnorm, we disabled these parameters in all architectures. To at least adhere to \cref{prescription:norm} at initialisation, AGD draws initial weight matrices uniform semi-orthogonal and re-scaled by a factor of $\sqrt{\mathtt{fan\_in}/\mathtt{fan\_out}}$. Adam and SGD baselines used the PyTorch default initialisation. A PyTorch implementation of AGD reflecting these details is given in \cref{app:pytorch}. All experiments use square loss except ImageNet which used cross-entropy loss. Cross-entropy loss has been found to be superior to square loss for datasets with a large number of classes \citep{Demirkaya2020ExploringTR,HuiSquareCrossEntropy}.


Our experimental results are spread across five figures:
\begin{itemize}[leftmargin=*]
    \item \cref{fig:showcase} presents some highlights of our results: First, AGD can train networks that Adam and SGD with default hyperparameters cannot. Second, for ResNet-18 on CIFAR-10, AGD attained performance comparable to the best-tuned performance of Adam and SGD. And third, AGD scales up to ImageNet.
    \item \cref{fig:1} displays the breadth of our experiments: from training a 16-layer fully-connected network on CIFAR-10 to training ResNet-50 on ImageNet. Adam's learning rate was tuned over the logarithmic grid $\{10^{-5},10^{-4},...,10^{-1}\}$ while for ImageNet we used a default learning rate of 0.1 for SGD without any manual decay. AGD and Adam performed almost equally well on the depth-16 width-512 fully-connected network: 52.7\% test accuracy for AGD compared to 53.5\% for Adam.
    For ResNet-18 on CIFAR-10, Adam attained  92.9\% test accuracy compared to AGD's 91.2\%. On this benchmark, a fully-tuned SGD with learning rate schedule, weight decay, cross-entropy loss and bias and affine parameters can attain 93.0\% test accuracy \citep{kuangliu}. For VGG-16 on CIFAR-100, AGD achieved 67.4\% test accuracy compared to Adam's 69.7\%.
    Finally, on ImageNet AGD achieved a top-1 test accuracy of 65.5\% after 350 epochs.
    \item \cref{fig:2} compares AGD to Adam and SGD for training an eight-layer fully-connected network of width 256. Adam and SGD's learning rates were tuned over the logarithmic grid $\{10^{-5},10^{-4},...,10^{-1}\}$. Adam's optimal learning rate of $10^{-4}$ was three orders of magnitude smaller than SGD's optimal learning rate of $10^{-1}$. SGD did not attain as low of an objective value as Adam or AGD.
    \item \cref{fig:3} shows that AGD can train FCNs with width ranging from 64 to 2048 and depth from 2 to 32 and \cref{fig:4} shows that AGD successfully trains a four-layer FCN at varying mini-batch size: from 32 to 4096.
\end{itemize}

\begin{figure}
    \centering
    \includegraphics[width=\textwidth]{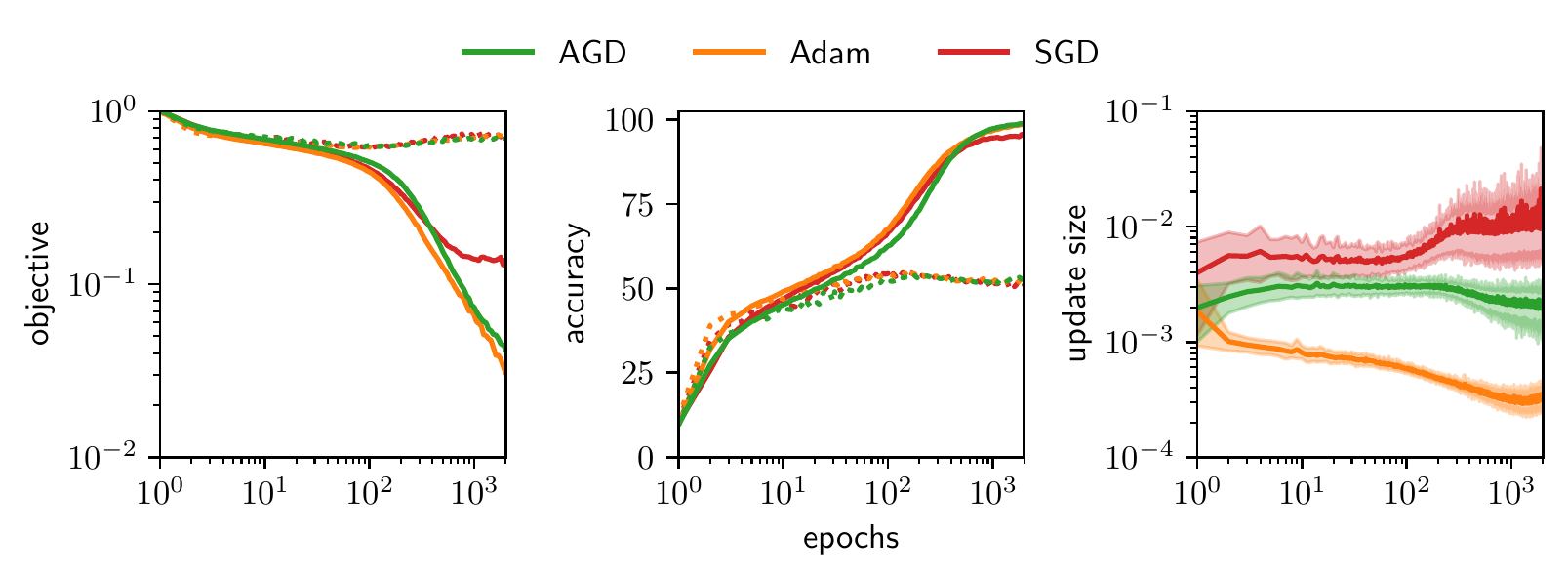}
    \caption{\captiontitle{Comparing automatic gradient descent to tuned Adam and SGD.} An eight-layer fully-connected network was trained on CIFAR-10 with square loss. Dotted lines show test and solid lines show train performance.
    The \captiontitle{left panel} shows the objective value: AGD and Adam attained a smaller training objective than SGD. The \captiontitle{middle panel} shows train and test accuracies. The \captiontitle{right panel} shows the relative update size averaged over layers: $\tfrac{1}{L}\sum_{k=1}^L \norm{\Delta \mW_k}_F/\norm{\mW_k}_F$. We plot the maximum, minimum and mean over an epoch.} \label{fig:2}
\end{figure}
\begin{figure}
    \centering
    \includegraphics[width=\textwidth]{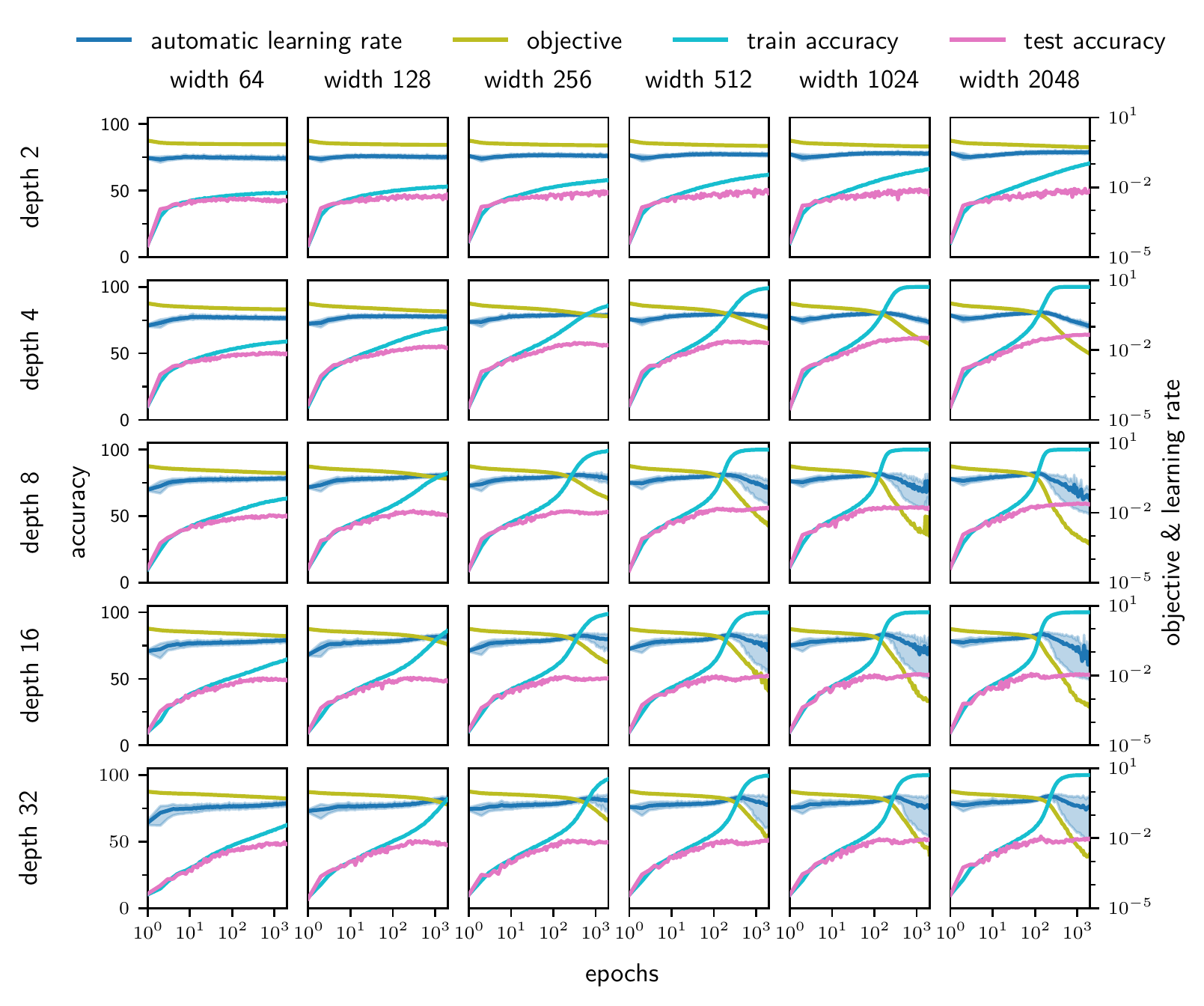}
    \caption{\captiontitle{Benchmarking automatic gradient descent on networks of varying width and depth.} We trained fully-connected networks on CIFAR-10 with square loss and a mini-batch size of 128. The depth ranged from $2$ to $32$, and the width from $64$ to $2048$, in powers of two. In terms of training performance, wider was always better, while depth 8 and depth 16 were superior to depth 32. In terms of test accuracy, the best performance was achieved at depth 4 and width 2048: 63.7\%. The worst test performance was achieved by the smallest network of depth 2 and width 64: 42.55\%.
    Larger networks display two broadly distinct phases of training: the automatic learning rate increases slowly while the objective decreases slowly, followed by a rapid decrease in the automatic learning rate and objective. This second phase typically coincides with reaching 100\% train accuracy. See \cref{fig:2} for a comparison between Adam, SGD and AGD for the 256-width 8-layer FCN.} \label{fig:3}
\end{figure}

\begin{figure}
    \centering
    \includegraphics[width=\textwidth]{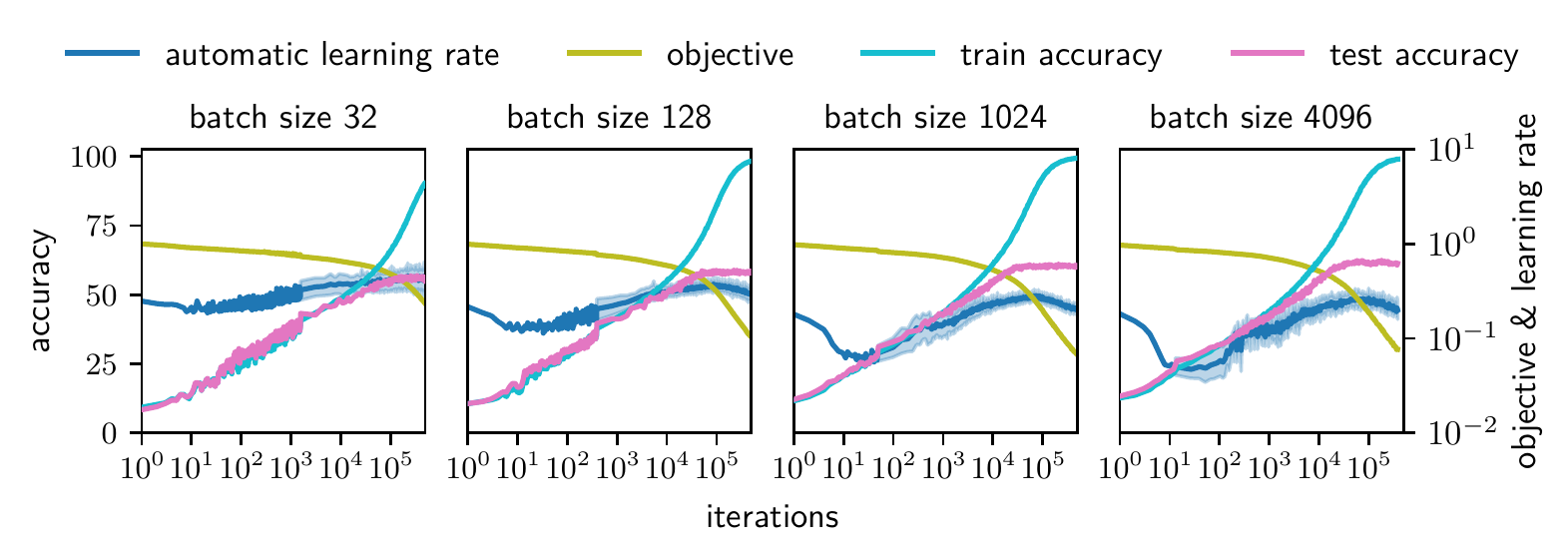}
    \caption{\captiontitle{Benchmarking automatic gradient descent at varying mini-batch size.} We trained four-layer fully-connected networks on CIFAR-10. The mini-batch size ranged from 32 to 4096. Test accuracy generally improved with increasing mini-batch size: the final test accuracies, in order of increasing mini-batch size, were 55.0\%, 58.0\%, 60.0\% and 59.8\%. The automatic learning rate seemed to initially dip, and this effect was more pronounced for larger mini-batch sizes. Metrics were computed every iteration during the first epoch and once per epoch from thereon---this explains the kinks visible in the plots.
    } \label{fig:4}
\end{figure}

\section{Discussion}

This paper has proposed a new framework for deriving optimisation algorithms for non-convex composite objective functions, which are particularly prevalent in the field of machine learning and the subfield of deep learning. What we have proposed is truly a \textit{framework}: it can be applied to a new loss function by writing down its Bregman divergence, or a new machine learning model by writing down its architectural perturbation bound. The framework is properly placed in the context of existing frameworks such as the majorise-minimise meta-algorithm, mirror descent and natural gradient descent.

Recent papers have proposed a paradigm of \textit{hyperparameter transfer} where a small network is tuned and the resulting hyperparameters are transferred to a larger network \citep{yang2021tuning, bernstein-thesis}. The methods and results in this paper suggest a stronger paradigm of \textit{hyperparameter elimination}: by detailed analysis of the structure and interactions between different components of a machine learning system, we may hope---if not to outright outlaw hyperparameters---at least to reduce their abundance and opacity.

The main product of this research is automatic gradient descent (AGD), with pseudocode given in \cref{alg:agd} and PyTorch code given in \cref{app:pytorch}. We have found AGD to be genuinely useful, and believe that it may complement automatic differentiation in helping to automate general machine learning workflows.

The analysis leading to automatic gradient descent is elementary: we leverage basic concepts in linear algebra such as matrix and vector norms, and use simple bounds such as the triangle inequality for vector--vector sums, and the operator norm bound for matrix--vector products. The analysis is non-asymptotic: it does not rely on taking dimensions to infinity, and deterministic: it does not involve random matrix theory. We believe that the accessibility of the analysis could make this paper a good starting point for future developments.

\paragraph{Directions for future work} Here we list some promising avenues for theoretical and practical research. We are exploring some of these ideas in our development codebase: \url{https://github.com/C1510/agd_exp}.

\begin{itemize}[leftmargin=*]
    \item \textit{Stochastic optimisation.} Automatic gradient descent is derived in the full-batch optimisation setting, but the algorithm is evaluated experimentally in the mini-batch setting. It would be interesting to try to extend our theoretical and practical methods to more faithfully address stochastic optimisation.
    \item \textit{More architectures.} Automatic gradient descent is derived for fully-connected networks and extended heuristically to convolutional networks. We are curious to extend the methods to more varied architectures such as transformers \citep{NIPS2017_3f5ee243} and architectural components such as biases. Since most neural networks resemble fully-connected networks in the sense that they are all just deep compound operators, we expect much of the structure of automatic gradient descent as presented to carry through.
    \item \textit{Regularisation.} The present paper deals purely with the optimisation structure of deep neural networks, and little thought is given to either generalisation or regularisation. Future work could look at both theoretical and practical regularisation schemes for automatic gradient descent. It would be interesting to try to do this without introducing hyperparameters, although we suspect that when it comes to regularisation at least one hyperparameter may become necessary.
    \item \textit{Acceleration.} We have found in some preliminary experiments that slightly increasing the update size of automatic gradient descent with a gain hyperparameter, or introducing a momentum hyperparameter, can lead to faster convergence. We emphasise that no experiment in this paper used such hyperparameters. Still, these observations may provide a valuable starting point for improving AGD in future work.
    \item \textit{Operator perturbation theory.} Part of the inspiration for this paper was the idea of applying operator perturbation theory to deep learning. While perturbation theory is well-studied in the context of linear operators \citep{Weyl1912,Kato:1966:PTL,STEWART200653}, in deep learning we are concerned with non-linear compound operators. It may be interesting to try to further extend results in perturbation theory to deep neural networks. One could imagine cataloging the perturbation structure of different neural network building blocks, and using a result similar to deep relative trust (\cref{lem:deep_perturbation_bounds}) to describe how they compound.
\end{itemize}
\subsubsection*{Acknowledgments} 

The authors are grateful to MIT SuperCloud, Oxford Hydra, NVIDIA and Virgile Richard for providing GPUs. Thanks are due to Greg Yang and Jamie Simon for helpful discussions. A paper with Greg and Jamie is in preparation to explain the relationship between muP \citep{Yang2021TensorPI} and the operator norm.

\bibliography{refs}

\begin{thebibliography}{51}
\providecommand{\natexlab}[1]{#1}
\providecommand{\url}[1]{\texttt{#1}}
\expandafter\ifx\csname urlstyle\endcsname\relax
  \providecommand{\doi}[1]{doi: #1}\else
  \providecommand{\doi}{doi: \begingroup \urlstyle{rm}\Url}\fi

\bibitem[Agarwal et~al.(2016)Agarwal, Zhu, Bullins, Hazan, and
  Ma]{Agarwal2016FindingAL}
Naman Agarwal, Zeyuan~Allen Zhu, Brian Bullins, Elad Hazan and Tengyu Ma.
\newblock Finding approximate local minima faster than gradient descent.
\newblock \emph{Symposium on Theory of Computing}, 2016.

\bibitem[Agarwal et~al.(2017)Agarwal, Bullins, and Hazan]{hessian-linear}
Naman Agarwal, Brian Bullins and Elad Hazan.
\newblock Second-order stochastic optimization for machine learning in linear
  time.
\newblock \emph{Journal of Machine Learning Research}, 2017.

\bibitem[Al-Rfou et~al.(2016)Al-Rfou, Alain, Almahairi, Angermueller, Bahdanau,
  Ballas, Bastien, Bayer, Belikov, Belopolsky, Bengio, Bergeron, Bergstra,
  Bisson, {Bleecher Snyder}, Bouchard, Boulanger-Lewandowski, Bouthillier,
  de~Br\'ebisson, Breuleux, Carrier, Cho, Chorowski, Christiano, Cooijmans,
  C\^ot\'e, C\^ot\'e, Courville, Dauphin, Delalleau, Demouth, Desjardins,
  Dieleman, Dinh, Ducoffe, Dumoulin, {Ebrahimi Kahou}, Erhan, Fan, Firat,
  Germain, Glorot, Goodfellow, Graham, Gulcehre, Hamel, Harlouchet, Heng,
  Hidasi, Honari, Jain, Jean, Jia, Korobov, Kulkarni, Lamb, Lamblin, Larsen,
  Laurent, Lee, Lefrancois, Lemieux, L\'eonard, Lin, Livezey, Lorenz, Lowin,
  Ma, Manzagol, Mastropietro, McGibbon, Memisevic, van Merri\"enboer,
  Michalski, Mirza, Orlandi, Pal, Pascanu, Pezeshki, Raffel, Renshaw, Rocklin,
  Romero, Roth, Sadowski, Salvatier, Savard, Schl\"uter, Schulman, Schwartz,
  Serban, Serdyuk, Shabanian, Simon, Spieckermann, Subramanyam, Sygnowski,
  Tanguay, van Tulder, Turian, Urban, Vincent, Visin, de~Vries, Warde-Farley,
  Webb, Willson, Xu, Xue, Yao, Zhang, and Zhang]{theano}
Rami Al-Rfou, Guillaume Alain, Amjad Almahairi, Christof Angermueller, Dzmitry
  Bahdanau, Nicolas Ballas, Fr\'ed\'eric Bastien, Justin Bayer, Anatoly
  Belikov, Alexander Belopolsky et~al.
\newblock {Theano: A {Python} framework for fast computation of mathematical
  expressions}.
\newblock \emph{arXiv:1605.02688}, 2016.

\bibitem[Amari(1998)]{amari}
Shun-ichi Amari.
\newblock Natural gradient works efficiently in learning.
\newblock \emph{Neural Computation}, 1998.

\bibitem[Azizan \& Hassibi(2019)Azizan and Hassibi]{azizan2018stochastic}
Navid Azizan and Babak Hassibi.
\newblock Stochastic gradient/mirror descent: {M}inimax optimality and implicit
  regularization.
\newblock In \emph{International Conference on Learning Representations}, 2019.

\bibitem[Bernstein(2022)]{bernstein-thesis}
Jeremy Bernstein.
\newblock \emph{Optimisation \& Generalisation in Networks of Neurons}.
\newblock {Ph.D.} thesis, California Institute of Technology, 2022.

\bibitem[Bernstein et~al.(2020)Bernstein, Vahdat, Yue, and Liu]{my-fromage}
Jeremy Bernstein, Arash Vahdat, Yisong Yue and Ming-Yu Liu.
\newblock On the distance between two neural networks and the stability of
  learning.
\newblock In \emph{Neural Information Processing Systems}, 2020.

\bibitem[Björck(1996)]{gauss-newton}
Åke Björck.
\newblock \emph{Numerical Methods for Least Squares Problems}.
\newblock Society for Industrial and Applied Mathematics, 1996.

\bibitem[Bottou et~al.(2018)Bottou, Curtis, and Nocedal]{bottou}
L\'{e}on Bottou, Frank~E. Curtis and Jorge Nocedal.
\newblock Optimization methods for large-scale machine learning.
\newblock \emph{SIAM Review}, 2018.

\bibitem[Bregman(1967)]{bregman1967relaxation}
Lev~M. Bregman.
\newblock The relaxation method of finding the common point of convex sets and
  its application to the solution of problems in convex programming.
\newblock \emph{USSR Computational Mathematics and Mathematical Physics}, 1967.

\bibitem[Brock et~al.(2021)Brock, De, Smith, and Simonyan]{pmlr-v139-brock21a}
Andy Brock, Soham De, Samuel~L. Smith and Karen Simonyan.
\newblock High-performance large-scale image recognition without normalization.
\newblock In \emph{International Conference on Machine Learning}, 2021.

\bibitem[Carbonnelle \& Vleeschouwer(2019)Carbonnelle and
  Vleeschouwer]{carbonnelle2019layer}
Simon Carbonnelle and Christophe~De Vleeschouwer.
\newblock Layer rotation: {A} surprisingly simple indicator of generalization
  in deep networks?
\newblock In \emph{ICML Workshop on Identifying and Understanding Deep Learning
  Phenomena}, 2019.

\bibitem[Cohen et~al.(2021)Cohen, Kaur, Li, Kolter, and
  Talwalkar]{cohen2021gradient}
Jeremy Cohen, Simran Kaur, Yuanzhi Li, J.~Zico Kolter and Ameet Talwalkar.
\newblock Gradient descent on neural networks typically occurs at the edge of
  stability.
\newblock In \emph{International Conference on Learning Representations}, 2021.

\bibitem[Demirkaya et~al.(2020)Demirkaya, Chen, and
  Oymak]{Demirkaya2020ExploringTR}
Ahmet Demirkaya, Jiasi Chen and Samet Oymak.
\newblock Exploring the role of loss functions in multiclass classification.
\newblock \emph{Conference on Information Sciences and Systems}, 2020.

\bibitem[Deng et~al.(2009)Deng, Dong, Socher, Li, Li, and
  Fei-Fei]{deng2009imagenet}
Jia Deng, Wei Dong, Richard Socher, Li-Jia Li, Kai Li and Li~Fei-Fei.
\newblock {ImageNet}: A large-scale hierarchical image database.
\newblock In \emph{Computer Vision and Pattern Recognition}, 2009.

\bibitem[Dhillon \& Tropp(2008)Dhillon and Tropp]{bregman}
Inderjit~S. Dhillon and Joel~A. Tropp.
\newblock Matrix nearness problems with {Bregman} divergences.
\newblock \emph{SIAM Journal on Matrix Analysis and Applications}, 2008.

\bibitem[Farhang et~al.(2022)Farhang, Bernstein, Tirumala, Liu, and
  Yue]{my-margin}
Alexander~R. Farhang, Jeremy Bernstein, Kushal Tirumala, Yang Liu and Yisong
  Yue.
\newblock Investigating generalization by controlling normalized margin.
\newblock In \emph{International Conference on Machine Learning}, 2022.

\bibitem[Goodfellow et~al.(2016)Goodfellow, Bengio, and
  Courville]{Goodfellow-et-al-2016}
Ian Goodfellow, Yoshua Bengio and Aaron Courville.
\newblock \emph{Deep Learning}.
\newblock MIT Press, 2016.

\bibitem[He et~al.(2015)He, Zhang, Ren, and Sun]{He2015DeepRL}
Kaiming He, X.~Zhang, Shaoqing Ren and Jian Sun.
\newblock Deep residual learning for image recognition.
\newblock \emph{Computer Vision and Pattern Recognition}, 2015.

\bibitem[Henderson et~al.(2018)Henderson, Islam, Bachman, Pineau, Precup, and
  Meger]{deeprlmatters}
Peter Henderson, Riashat Islam, Philip Bachman, Joelle Pineau, Doina Precup and
  David Meger.
\newblock Deep reinforcement learning that matters.
\newblock In \emph{AAAI Conference on Artificial Intelligence}, 2018.

\bibitem[Hui \& Belkin(2021)Hui and Belkin]{HuiSquareCrossEntropy}
Like Hui and Mikhail Belkin.
\newblock Evaluation of neural architectures trained with square loss vs.\
  cross-entropy in classification tasks.
\newblock In \emph{International Conference on Learning Representations}, 2021.

\bibitem[Jacot et~al.(2018)Jacot, Gabriel, and Hongler]{NTKjacot}
Arthur Jacot, Franck Gabriel and Clement Hongler.
\newblock Neural tangent kernel: {C}onvergence and generalization in neural
  networks.
\newblock In \emph{Neural Information Processing Systems}, 2018.

\bibitem[Jiang et~al.(2020)Jiang, Neyshabur, Mobahi, Krishnan, and
  Bengio]{jiang2019fantastic}
Yiding Jiang, Behnam Neyshabur, Hossein Mobahi, Dilip Krishnan and Samy Bengio.
\newblock Fantastic generalization measures and where to find them.
\newblock In \emph{International Conference on Learning Representations}, 2020.

\bibitem[Kato(1966)]{Kato:1966:PTL}
Tosio Kato.
\newblock \emph{Perturbation Theory for Linear Operators}.
\newblock Springer, 1966.

\bibitem[Kingma \& Ba(2015)Kingma and Ba]{kingma_adam:_2015}
Diederik~P. Kingma and Jimmy Ba.
\newblock Adam: {A} method for stochastic optimization.
\newblock In \emph{International Conference on Learning Representations}, 2015.

\bibitem[Krizhevsky(2009)]{Krizhevsky09learningmultiple}
Alex Krizhevsky.
\newblock Learning multiple layers of features from tiny images.
\newblock Technical report, University of Toronto, 2009.

\bibitem[Lange(2016)]{mm}
Kenneth Lange.
\newblock \emph{{MM} Optimization Algorithms}.
\newblock Society for Industrial and Applied Mathematics, 2016.

\bibitem[Lee et~al.(2019)Lee, Xiao, Schoenholz, Bahri, Novak, Sohl-Dickstein,
  and Pennington]{NEURIPS2019_0d1a9651}
Jaehoon Lee, Lechao Xiao, Samuel Schoenholz, Yasaman Bahri, Roman Novak, Jascha
  Sohl-Dickstein and Jeffrey Pennington.
\newblock Wide neural networks of any depth evolve as linear models under
  gradient descent.
\newblock In \emph{Neural Information Processing Systems}, 2019.

\bibitem[Liu et~al.(2022)Liu, Zhu, and Belkin]{LIU202285}
Chaoyue Liu, Libin Zhu and Mikhail Belkin.
\newblock Loss landscapes and optimization in over-parameterized non-linear
  systems and neural networks.
\newblock \emph{Applied and Computational Harmonic Analysis}, 2022.

\bibitem[Liu(2017)]{kuangliu}
Kuang Liu.
\newblock Train {CIFAR-10} with {PyTorch}.
\newblock \url{https://github.com/kuangliu/pytorch-cifar}, 2017.

\bibitem[Loewenstein et~al.(2011)Loewenstein, Kuras, and
  Rumpel]{Loewenstein9481}
Yonatan Loewenstein, Annerose Kuras and Simon Rumpel.
\newblock Multiplicative dynamics underlie the emergence of the log-normal
  distribution of spine sizes in the neocortex in vivo.
\newblock \emph{Journal of Neuroscience}, 2011.

\bibitem[Lucic et~al.(2017)Lucic, Kurach, Michalski, Gelly, and
  Bousquet]{Lucic2017AreGC}
Mario Lucic, Karol Kurach, Marcin Michalski, Sylvain Gelly and Olivier
  Bousquet.
\newblock Are {GANs} created equal? {A} large-scale study.
\newblock In \emph{Neural Information Processing Systems}, 2017.

\bibitem[Nemirovsky \& Yudin(1983)Nemirovsky and Yudin]{nemirovsky_yudin_1983}
Arkady~S. Nemirovsky and David~B. Yudin.
\newblock \emph{Problem complexity and method efficiency in optimization}.
\newblock Wiley, 1983.

\bibitem[Nesterov \& Polyak(2006)Nesterov and Polyak]{Nesterov2006CubicRO}
Yurii Nesterov and Boris Polyak.
\newblock Cubic regularization of {N}ewton method and its global performance.
\newblock \emph{Mathematical Programming}, 2006.

\bibitem[Nocedal \& Wright(1999)Nocedal and Wright]{Nocedal1999NumericalO}
Jorge Nocedal and Stephen~J. Wright.
\newblock \emph{Numerical Optimization}.
\newblock Springer, 1999.

\bibitem[Orabona \& Cutkosky(2020)Orabona and Cutkosky]{tutorial}
Francesco Orabona and Ashok Cutkosky.
\newblock {ICML} 2020 tutorial on parameter-free online optimization, 2020.

\bibitem[Pascanu \& Bengio(2014)Pascanu and Bengio]{revisiting-ngd}
Razvan Pascanu and Yoshua Bengio.
\newblock Revisiting natural gradient for deep networks.
\newblock In \emph{International Conference on Learning Representations}, 2014.

\bibitem[Paszke et~al.(2019)Paszke, Gross, Massa, Lerer, Bradbury, Chanan,
  Killeen, Lin, Gimelshein, Antiga, Desmaison, Kopf, Yang, DeVito, Raison,
  Tejani, Chilamkurthy, Steiner, Fang, Bai, and Chintala]{pytorch}
Adam Paszke, Sam Gross, Francisco Massa, Adam Lerer, James Bradbury, Gregory
  Chanan, Trevor Killeen, Zeming Lin, Natalia Gimelshein, Luca Antiga et~al.
\newblock {PyTorch}: An imperative style, high-performance deep learning
  library.
\newblock In \emph{Neural Information Processing Systems}, 2019.

\bibitem[Philipp et~al.(2017)Philipp, Song, and Carbonell]{Philipp2017TheEG}
George Philipp, Dawn~Xiaodong Song and Jaime~G. Carbonell.
\newblock The exploding gradient problem demystified.
\newblock \emph{arXiv:1712.05577}, 2017.

\bibitem[Rumelhart et~al.(1986)Rumelhart, Hinton, and
  Williams]{Rumelhart1986LearningRB}
David~E. Rumelhart, Geoffrey~E. Hinton and Ronald~J. Williams.
\newblock Learning representations by back-propagating errors.
\newblock \emph{Nature}, 1986.

\bibitem[Schmidt et~al.(2021)Schmidt, Schneider, and Hennig]{crowded_valley}
Robin~M. Schmidt, Frank Schneider and Philipp Hennig.
\newblock Descending through a crowded valley---benchmarking deep learning
  optimizers.
\newblock In \emph{International Conference on Machine Learning}, 2021.

\bibitem[Sharir et~al.(2020)Sharir, Peleg, and Shoham]{Sharir2020TheCO}
Or~Sharir, Barak Peleg and Yoav Shoham.
\newblock The cost of training {NLP} models: A concise overview.
\newblock \emph{arXiv:2004.08900}, 2020.

\bibitem[Simonyan \& Zisserman(2015)Simonyan and Zisserman]{simonyan2015a}
Karen Simonyan and Andrew Zisserman.
\newblock Very deep convolutional networks for large-scale image recognition.
\newblock In \emph{International Conference on Learning Representations}, 2015.

\bibitem[Stewart(2006)]{STEWART200653}
Michael Stewart.
\newblock Perturbation of the {SVD} in the presence of small singular values.
\newblock \emph{Linear Algebra and its Applications}, 2006.

\bibitem[Sun et~al.(2022)Sun, Ahn, Thrampoulidis, and Azizan]{sun2022mirror}
Haoyuan Sun, Kwangjun Ahn, Christos Thrampoulidis and Navid Azizan.
\newblock Mirror descent maximizes generalized margin and can be implemented
  efficiently.
\newblock In \emph{Neural Information Processing Systems}, 2022.

\bibitem[Vaswani et~al.(2017)Vaswani, Shazeer, Parmar, Uszkoreit, Jones, Gomez,
  Kaiser, and Polosukhin]{NIPS2017_3f5ee243}
Ashish Vaswani, Noam Shazeer, Niki Parmar, Jakob Uszkoreit, Llion Jones,
  Aidan~N Gomez, \L{}ukasz Kaiser and Illia Polosukhin.
\newblock Attention is all you need.
\newblock In \emph{Neural Information Processing Systems}, 2017.

\bibitem[Weyl(1912)]{Weyl1912}
Hermann Weyl.
\newblock Das asymptotische {V}erteilungsgesetz der {E}igenwerte linearer
  partieller {D}ifferentialgleichungen (mit einer {A}nwendung auf die {T}heorie
  der {H}ohlraumstrahlung).
\newblock \emph{Mathematische Annalen}, 1912.

\bibitem[Yang \& Hu(2021)Yang and Hu]{Yang2021TensorPI}
Greg Yang and Edward~J. Hu.
\newblock Tensor programs {IV}: Feature learning in infinite-width neural
  networks.
\newblock In \emph{International Conference on Machine Learning}, 2021.

\bibitem[Yang et~al.(2021)Yang, Hu, Babuschkin, Sidor, Liu, Farhi, Ryder,
  Pachocki, Chen, and Gao]{yang2021tuning}
Greg Yang, Edward~J. Hu, Igor Babuschkin, Szymon Sidor, Xiaodong Liu, David
  Farhi, Nick Ryder, Jakub Pachocki, Weizhu Chen and Jianfeng Gao.
\newblock Tuning large neural networks via zero-shot hyperparameter transfer.
\newblock In \emph{Neural Information Processing Systems}, 2021.

\bibitem[You et~al.(2017)You, Gitman, and Ginsburg]{You:EECS-2017-156}
Yang You, Igor Gitman and Boris Ginsburg.
\newblock Scaling {SGD} batch size to 32{K} for {I}mage{N}et training.
\newblock Technical report, University of California, Berkeley, 2017.

\bibitem[Zhang et~al.(2020)Zhang, He, Sra, and Jadbabaie]{Zhang2020Why}
Jingzhao Zhang, Tianxing He, Suvrit Sra and Ali Jadbabaie.
\newblock Why gradient clipping accelerates training: A theoretical
  justification for adaptivity.
\newblock In \emph{International Conference on Learning Representations}, 2020.

\end{thebibliography}
\bibliographystyle{tmlr/tmlr}

\newpage
\appendix
\section{Proofs}
\label{app:proofs}

Here are the proofs for the theoretical results in the main text.

\decomposition*
\begin{proof}[\mbox{\hyperref[thm:decomposition]{Proof}}]\label{proof:decomposition}
By the chain rule, $\nabla_\vw\el(\vw)^\top \Delta \vw = \frac{1}{|\set{S}|}\sum_{(\vx,\vy)\in \set{S}} \nabla_{\vf(\vx)} \ell(\vf(\vx),\vy)^\top \nabla_\vw \vf(\vx) \Delta \vw$. Therefore:
\begin{equation*}
    \Delta \el(\vw) - \nabla_\vw\el(\vw)^\top \Delta \vw = \frac{1}{|\set{S}|}\sum_{(\vx,\vy)\in \set{S}}\Delta \ell(\vf(\vx), \vy) - \nabla_{\vf(\vx)} \ell(\vf(\vx),\vy)^\top \nabla_\vw \vf(\vx) \Delta \vw.
\end{equation*}
Adding and subtracting $\frac{1}{|\set{S}|}\sum_{(\vx,\vy)\in \set{S}}\nabla_{\vf(\vx)}\ell(\vf(\vx),\vy)^\top \Delta \vf(\vx)$ on the right-hand side yields the result.
\end{proof}

\squarebreg*
\begin{proof}[\mbox{\hyperref[lem:sq-bregman]{Proof}}]\label{proof:squarebreg}
Expanding the Euclidean norms in the loss perturbation $\Delta \ell$ yields:
\begin{align*}
    \Delta \ell(\vf(\vx), \vy) & = \tfrac{1}{2d_L} \norm{\vf(\vx) + \Delta \vf(\vx) - \vy}_2^2 - \tfrac{1}{2d_L} \norm{\vf(\vx) - \vy}_2^2 \\
    &= \tfrac{1}{2d_L} \norm{\Delta \vf(\vx)}_2^2 + (\vf(\vx) - \vy)^\top \Delta \vf(\vx).
\end{align*}
The result follows by identifying that $\nabla_{\vf(\vx)}\ell(\vf(\vx),\vy)^\top \Delta \vf(\vx) = (\vf(\vx) - \vy)^\top \Delta \vf(\vx)$.
\end{proof}

\xentbreg*
\begin{proof}[\mbox{\hyperref[lem:xent-bregman]{Proof}}]\label{proof:xentbreg}
First, since $\sum_i \vy_i =1$, cross-entropy loss may be re-written:
\begin{align*}
    \ell(\vf(\vx), \vy) \defeq - \log [\softmax(\vf(\vx))]^\top \vy = - \vf(\vx)^\top \vy +  \log \norm{\exp \vf(\vx)}_1.
\end{align*}
The linear term $- \vf(\vx)^\top \vy$ does not contribute to the linearisation error and may be neglected. Therefore:
\begin{align*}
    &\Delta \ell(\vf(\vx), \vy) -\nabla_{\vf(\vx)}\ell(\vf(\vx),\vy)^\top \Delta \vf(\vx) \\
    &\quad\quad= \log \norm{\exp (\vf(\vx)+\Delta \vf(\vx))}_1 - \log \norm{\exp \vf(\vx)}_1 - \nabla_{\vf(\vx)}\log \norm{\exp \vf(\vx)}_1^\top \Delta \vf(\vx) \\
    &\quad\quad= \log \frac{1/\norm{\exp \vf(\vx)}_1}{1/\norm{\exp (\vf(\vx)+\Delta \vf(\vx))}_1} - \frac{\exp\vf(\vx)^\top}{\norm{\exp \vf(\vx)}_1} \Delta \vf(\vx)\\
    &\quad\quad=\frac{\exp\vf(\vx)^\top}{\norm{\exp \vf(\vx)}_1} \log \frac{\exp \vf(\vx)/\norm{\exp \vf(\vx)}_1}{\exp (\vf(\vx)+\Delta \vf(\vx))/\norm{\exp (\vf(\vx)+\Delta \vf(\vx))}_1}.
\end{align*}
The final line is equivalent to $\kl \Big(\softmax(\vf(\vx))\,\Big|\Big|\, \softmax(\vf(\vx)+\Delta \vf(\vx))\Big)$ establishing the first equality.

To establish the inequality, let $\otimes$ denote the outer product and define $p \defeq\softmax(f(\vx))$. Then we have:
\begin{align*}
    \Delta \ell(\vf(\vx), \vy) -\nabla_{\vf(\vx)}\ell(\vf(\vx),\vy)^\top \Delta \vf(\vx) &= \frac{1}{2}\Delta \vf(\vx)^\top \nabla^2_{\vf(\vx)}\ell(\vf(\vx), \vy) \Delta \vf(\vx) + \mathcal{O}(\Delta \vf^3) \\
    &= \frac{1}{2}\Delta \vf(\vx)^\top \nabla^2_{\vf(\vx)}\log \norm{\exp \vf(\vx)}_1 \Delta \vf(\vx) + \mathcal{O}(\Delta \vf^3)\\
    &= \frac{1}{2}\Delta \vf(\vx)^\top [\diag (p) - p \otimes p] \Delta \vf(\vx) + \mathcal{O}(\Delta \vf^3)\\
    &\leq \frac{1}{2}\Delta \vf(\vx)^\top \diag (p) \Delta \vf(\vx) + \mathcal{O}(\Delta \vf^3)\\
    &\leq \frac{1}{2}\norm{\Delta \vf(\vx)}_\infty^2 + \mathcal{O}(\Delta \vf^3),
\end{align*}
where we have used that $p\otimes p$ is positive definite and then applied H\"older's inequality with $\norm{p}_1 = 1$.
\end{proof}

\functmajor*
\begin{proof}[\mbox{\hyperref[thm:functmajor]{Proof}}]\label{proof:functmajor}
The result follows by substituting \cref{ass:orthog} into \cref{thm:decomposition} and applying \cref{def:bregman}.
\end{proof}

\sqmajor*
\begin{proof}[\mbox{\hyperref[lem:sq-major]{Proof}}]\label{proof:sqmajor} Combine \cref{lem:sq-bregman} with \cref{thm:functmajor} to obtain the result.
\end{proof}

\xentmajor*
\begin{proof}[\mbox{\hyperref[lem:xent-major]{Proof}}]\label{proof:xentmajor} Combine \cref{lem:xent-bregman} with \cref{thm:functmajor} to obtain the result.
\end{proof}

\outbound*
\begin{proof}[\mbox{\hyperref[lem:outbound]{Proof}}]\label{proof:outbound}
For any vector $\vv$ and matrix $\mM$ with compatible dimensions, we have that $\norm{\mM \vv}_2 \leq \norm{\mM}_* \cdot \norm{\vv}_2$ and $\norm{\relu \vv}_2 \leq \norm{\vv}_2$. The lemma follows by applying these results recursively over the depth of the network.
\end{proof}

\archbounds*
\begin{proof}[\mbox{\hyperref[lem:deep_perturbation_bounds]{Proof}}]\label{proof:archbounds} We proceed by induction. First, consider a network with $L=1$ layers: $\vf(\vx) = \mW_1 \vx$. Observe that $\norm{\Delta \vf(\vx)}_2 = \norm{\Delta \mW_1 \vx}_2 \leq \norm{\Delta \mW_1}_*\cdot \norm{\vx}_2$ as required. Next, assume that the result holds for a network $\vg(\vx)$ with $L-1$ layers and consider adding a layer to obtain $\vf(\vx) = \mW_L\circ \relu{}\circ \vg(\vx)$. Then:
\begin{align*}
    \norm{\Delta \vf(\vx)}_2 &= \norm{(\mW_L+\Delta \mW_L)\circ \relu{} \circ (\vg(\vx)+\Delta \vg(\vx)) - \mW_L \circ \relu{} \circ \vg(\vx)}_2 \\ 
    &= \norm{\mW_L \left(\relu{} \circ (\vg(\vx)+\Delta
    \vg(\vx)) - \relu{} \circ \vg(\vx)\right) + \Delta \mW_L \left( \relu{} \circ (\vg(\vx)+\Delta \vg(\vx)) - \relu(0)\right)}_2 \\
    &\leq \norm{\mW_L}_*\cdot\norm{\Delta \vg(\vx)}_2 + \norm{\Delta \mW_L}_*\cdot(\norm{\vg(\vx)}_2 + \norm{\Delta \vg(\vx)}_2)\\
    &= (\norm{\mW_L}_*+\norm{\Delta \mW_L}_*)\cdot \norm{\Delta \vg(\vx)}_2 + \norm{\Delta \mW_L}_*\cdot \norm{\vg(\vx)}_2,
    \end{align*}
    where the inequality follows by applying the triangle inequality, the operator norm bound, the fact that $\relu{}$ is one-Lipschitz, and a further application of the triangle inequality. But by the inductive hypothesis and \cref{lem:outbound}, the right-hand side is bounded by:
    \begin{align*}
    (\norm{\mW_L}_*&+\norm{\Delta \mW_L}_*) \left[ \prod_{k = 1}^{L-1} \left( 1 + \frac{\Vert \Delta \mW_k \Vert_{*}}{\Vert \mW_k \Vert_{*}}\right)  - 1 \right] \times \left[\prod_{k=1}^{L-1} \norm{\mW_k}_* \right] \times \norm{\vx}_2 + \norm{\Delta \mW_L}_* \times \left[\prod_{k=1}^{L-1} \norm{\mW_k}_* \right] \times \norm{\vx}_2\\
    &= \left[ \prod_{k = 1}^L \left( 1 + \frac{\Vert \Delta \mW_k \Vert_{*}}{\Vert \mW_k \Vert_{*}}\right)  - 1 \right] \times \left[\prod_{k=1}^L \norm{\mW_k}_* \right] \times \norm{\vx}_2.
\end{align*}
The induction is complete. To further bound this result under \cref{prescription:norm}, observe that the product $\left[\prod_{k=1}^L \norm{\mW_k}_* \right] \times \norm{\vx}_2$ telescopes to just $\sqrt{d_L}$, while the other product satisfies:
\begin{equation*}
    \left[ \prod_{k = 1}^L \left( 1 + \frac{\Vert \Delta \mW_k \Vert_{*}}{\Vert \mW_k \Vert_{*}}\right)  - 1 \right] = \left(1+\frac{\eta}{L}\right)^L -1 \leq \lim_{L\to\infty}\left(1+\frac{\eta}{L}\right)^L-1 = \exp\eta - 1.
\end{equation*}
Combining these observations yields the result.
\end{proof}

\majordnn*
\begin{proof}[\mbox{\hyperref[lem:sq-major-nn]{Proof}}]\label{proof:majordnn}
Substitute \cref{lem:deep_perturbation_bounds} into \cref{lem:sq-major} and decompose $\nabla_\vw\el(\vw)^\top \Delta \vw = \sum_{k=1}^L \trace (\Delta \mW_k^\top \nabla_{\mW_k}\el)$. The result follows by realising that under \cref{prescription:norm}, the perturbations satisfy $\norm{\Delta \mW_k}_* = \sqrt{d_k/d_{k-1}} \cdot \frac{\eta}{L}$.
\end{proof}

\loglr*
\begin{proof}[\mbox{\hyperref[thm:log-lr]{Proof}}]\label{proof:loglr} The inner product $\trace\frac{\Delta \mW_k^\top\nabla_{\mW_k}\el}{\norm{\Delta \mW_k}_*}$ that appears in \cref{lem:sq-major-nn} is most negative when the perturbation $\Delta \mW_k$ satisfies $\Delta \mW_k/\norm{\Delta \mW_k}_* = - \nabla_{\mW_k}\el / \norm{\nabla_{\mW_k}\el}_*$. Substituting this result back into \cref{lem:sq-major-nn} yields:
\begin{equation*}
        \el(\vw+\Delta \vw) \leq \el(\vw) - \frac{\eta}{L}\sum_{k=1}^L\left[\sqrt{d_k/d_{k-1}} \times\frac{\norm{\nabla_{\mW_k}\el}_F^2}{\norm{\nabla_{\mW_k}\el}_*}\right] + \tfrac{1}{2} \,(\exp \eta -1)^2.
\end{equation*}
Under \cref{approx:g-cond}, we have that $\norm{\nabla_{\mW_k}\el}_F^2/\norm{\nabla_{\mW_k}\el}_* = \norm{\nabla_{\mW_k}\el}_F$ and so this inequality simplifies to:
\begin{equation*}
        \el(\vw+\Delta \vw) \leq \el(\vw) - \eta\cdot G + \tfrac{1}{2} \,(\exp \eta -1)^2.
\end{equation*}
Taking the derivative of the right-hand side with respect to $\eta$ and setting it to zero yields $(\exp\eta-1)\exp\eta = G$. Applying the quadratic formula and retaining the positive solution yields $\exp \eta = \half(1+\sqrt{1+4G})$. Combining this with the relation that $\Delta \mW_k/\norm{\Delta \mW_k}_* = - \nabla_{\mW_k}\el / \norm{\nabla_{\mW_k}\el}_*$ and applying that $\norm{\Delta \mW_k}_* = \sqrt{d_k/d_{k-1}} \cdot \frac{\eta}{L}$ by \cref{prescription:norm} yields the result.
\end{proof}

\objectivebound*
\begin{proof}[\mbox{\hyperref[lem:objectivebound]{Proof}}]\label{proof:objectivebound}
The result follows by the following chain of inequalities:
\begin{align*}
    \el(\vw) \defeq \frac{1}{\abs{\set{S}}} \sum_{(\vx,\vy)\in \set{S}}\frac{1}{2d_L}\norm{\vf(\vx;\vw) - \vy}_2^2 \leq \frac{1}{\abs{\set{S}}} \sum_{(\vx,\vy)\in \set{S}}\frac{1}{2d_L}(\norm{\vf(\vx;\vw)}_2^2 +\norm{\vy}_2^2) \leq \frac{1}{\abs{\set{S}}} \sum_{(\vx,\vy)\in \set{S}}\frac{d_L+d_L}{2d_L} = 1,
\end{align*}
where the second inequality holds under \cref{prescription:norm}.
\end{proof}

\gradientbound*
\begin{proof}[\mbox{\hyperref[lem:gradientbound]{Proof}}]\label{proof:gradientbound}
By the chain rule, the gradient of mean square error objective may be written:
\begin{align*}
    \nabla_{\mW_k} \el(\vw) = \frac{1}{\abs{\set{S}}} \sum_{(\vx,\vy)\in \set{S}}\frac{1}{d_L}(\vf(\vx;\vw) - \vy)^\top \mW_L \cdot \mD_{L-1}\mW_{L-1} \dots \mD_{k+1}\mW_{k+1} \cdot \mD_{k} \otimes \mD_{k-1} \mW_{k-1}\dots \mD_1 \mW_1 \vx,
\end{align*}
where $\otimes$ denotes the outer product and $\mD_k$ denotes a diagonal matrix whose entries are one when $\relu$ is active and zero when $\relu$ is inactive. Since the operator norm $\norm{\mD_k}_* = 1$, we have that the Frobenius norm $\norm{\nabla_{\mW_k} \el(\vw)}_F$ is bounded from above by:
\begin{align*}
    &\frac{1}{\abs{\set{S}}} \sum_{(\vx,\vy)\in \set{S}}\frac{1}{d_L}\norm{(\vf(\vx;\vw) - \vy)^\top \mW_L \cdot \mD_{L-1}\mW_{L-1} \dots \mD_{k+1}\mW_{k+1} \cdot \mD_{k} \otimes \mD_{k-1} \mW_{k-1}\dots \mD_1 \mW_1 \vx}_F\\
    &\hspace{3em}= \frac{1}{\abs{\set{S}}} \sum_{(\vx,\vy)\in \set{S}}\frac{1}{d_L}\norm{(\vf(\vx;\vw) - \vy)^\top \mW_L \cdot \mD_{L-1}\mW_{L-1} \dots \mD_{k+1}\mW_{k+1} \cdot \mD_{k}}_2 \cdot \norm{\mD_{k-1} \mW_{k-1}\dots \mD_1 \mW_1 \vx}_2\\
    &\hspace{3em}\leq \frac{1}{\abs{\set{S}}} \sum_{(\vx,\vy)\in \set{S}}\frac{1}{d_L}\norm{\vf(\vx;\vw) - \vy}_2\cdot \norm{\mW_L}_*\cdot \norm{\mW_{L-1}}_* \dots \norm{\mW_{k+1}}_*\cdot \norm{\mW_{k-1}}_*\dots \norm{\mW_1}_*\cdot \norm{\vx}_2 \\
    &\hspace{3em}= \frac{\prod_{l=1}^L\norm{\mW_l}_*}{\norm{\mW_k}} \times \frac{1}{\abs{\set{S}}} \sum_{(\vx,\vy)\in \set{S}}\frac{1}{d_L}\norm{\vf(\vx;\vw) - \vy}_2 \cdot \norm{\vx}_2 \\
    &\hspace{3em}\leq \frac{\prod_{l=1}^L\norm{\mW_l}_*}{\norm{\mW_k}_*} \cdot\frac{1}{\sqrt{d_L}} \sqrt{\frac{2}{\abs{\set{S}}} \sum_{(\vx,\vy)\in \set{S}}\frac{1}{2d_L}\norm{\vf(\vx;\vw) - \vy}_2^2} \cdot \sqrt{\frac{1}{\abs{\set{S}}} \sum_{(\vx,\vy)\in \set{S}}\norm{\vx}_2^2}\\
    &\hspace{3em}= \frac{\prod_{l=1}^L\norm{\mW_l}_*}{\norm{\mW_k}_*} \cdot \sqrt{\frac{2\el(\vw)}{d_L}} \cdot \sqrt{\frac{1}{\abs{\set{S}}} \sum_{(\vx,\vy)\in \set{S}}\norm{\vx}_2^2}.
\end{align*}
In the above argument, the first inequality follows by recursive application of the operator norm upper bound, and the second inequality follows from the Cauchy-Schwarz inequality. The right-hand side simplifies under \cref{prescription:norm}, and we may apply \cref{lem:objectivebound} to obtain:
\begin{align*}
    \norm{\nabla_{\mW_k} \el(\vw)}_F \leq \frac{\prod_{l=1}^L\norm{\mW_l}_*}{\norm{\mW_k}_*} \cdot \sqrt{\frac{2\el(\vw)}{d_L}} \cdot \sqrt{\frac{1}{\abs{\set{S}}} \sum_{(\vx,\vy)\in \set{S}}\norm{\vx}_2^2} \leq \frac{\sqrt{d_L/d_0}}{\sqrt{d_k / d_{k-1}}} \cdot \sqrt{\frac{2}{d_L}}\cdot \sqrt{d_0} = \sqrt{2}\cdot \sqrt{\frac{d_{k-1}}{d_k}}.
\end{align*}
\end{proof}

\criticalrate*
\begin{proof}[\mbox{\hyperref[lem:criticalrate]{Proof}}]\label{proof:criticalrate}
\cref{thm:log-lr} prescribes that $\exp\eta = \half(1+\sqrt{1+4G})$, and so $\eta = \log\big(1+\frac{\sqrt{1+4G}-1}{2}\big)$. We begin by proving some useful auxiliary bounds.  By \cref{lem:gradientbound} and \cref{prescription:norm}, the gradient summary is bounded by:
\begin{align*}
    G \defeq \frac{1}{L}\sum_{k=1}^L \sqrt{d_k/d_{k-1}} \cdot \norm{ \nabla_{\mW_k} \el(\vw)}_F \leq \frac{1}{L}\sum_{k=1}^L \sqrt{2} < 2.
\end{align*}
The fact that the gradient summary $G$ is less than two is important because, for $x\leq 1$, we have that $\log(1+x) \geq x \log 2$. In turn, this implies that since $G<2$, we have that $\eta = \log \frac{1+\sqrt{1+4G}}{2} \geq \frac{\sqrt{1+4G} - 1}{2} \log 2$. It will also be important to know that for $G<2$, we have that $\half\cdot G \leq \tfrac{\sqrt{1+4G} - 1}{2} \leq G$. 

With these bounds in hand, the analysis becomes fairly standard. By an intermediate step in the proof of \cref{thm:log-lr}, the change in objective across a single step is bounded by:
\begin{align*}
    \el(\vw+\Delta \vw)- \el(\vw)&\leq - \eta\cdot G + \tfrac{1}{2} \,(\exp \eta -1)^2 \\
    &\leq - \tfrac{\sqrt{1+4G} - 1}{2} (G \log 2 - \half  \tfrac{\sqrt{1+4G} - 1}{2})\\
    &\leq -\half \cdot (\log 2 - \half)\cdot G^2
    \leq -G^2 / 11,
\end{align*}
where the second and third inequalities follow by our auxiliary bounds. Letting $G_t$ denote the gradient summary at step $t$, averaging this bound over time steps and applying the telescoping property yields:
\begin{equation*}
    \min_{t\in[1,...,T]} G_t^2 \leq \frac{1}{T}\sum_{t=1}^{T} G_t^2 \leq \frac{11}{T}\sum_{t=1}^{T} \el(\vw_t) - \el(\vw_{t+1}) = \frac{11}{T}\cdot (\el(\vw_1) - \el(\vw_T)) \leq \frac{11}{T},
\end{equation*}
where the final inequality follows by \cref{lem:objectivebound} and the fact that $\el(\vw_T)\geq0$.

\end{proof}

\globalrate*
\begin{proof}[\mbox{\hyperref[thm:globalrate]{Proof}}]\label{proof:globalrate} 

By \cref{ass:pl}, the gradient summary at time step $t$ must satisfy $G_t \geq \alpha \times \sqrt{2\cdot\el(\vw_t)}$. Therefore the objective at time step $t$ is bounded by $\el(\vw_t) \leq G_t^2/(2\alpha^2)$. Combining with \cref{lem:criticalrate} then yields that:
\begin{equation*}
\el(\vw_T) = \min_{t\in[1,...,T]} \el(\vw_t) \leq \frac{1}{2\alpha^2}\min_{t\in[1,...,T]}G_t^2 \leq \frac{6}{\alpha^2T}.
\end{equation*}
The proof is complete.
\end{proof}

\newpage
\section{PyTorch Implementation}
\label{app:pytorch}

The following code implements automatic gradient descent in PyTorch \citep{pytorch}. We include a single gain hyperparameter which controls the update size and may be increased from its default value of 1.0 to slightly accelerate training. We emphasise that all the results reported in the paper used a gain of unity.

\begin{minted}{python}
import math
import torch

from torch.nn.init import orthogonal_

def singular_value(p):
    sv = math.sqrt(p.shape[0] / p.shape[1])
    if p.dim() == 4:
        sv /= math.sqrt(p.shape[2] * p.shape[3])
    return sv

class AGD:
    @torch.no_grad()
    def __init__(self, net, gain=1.0):

        self.net = net
        self.depth = len(list(net.parameters()))
        self.gain = gain

        for p in net.parameters():
            if p.dim() == 1: raise Exception("Biases are not supported.")
            if p.dim() == 2: orthogonal_(p)
            if p.dim() == 4:
                for kx in range(p.shape[2]):
                    for ky in range(p.shape[3]):
                        orthogonal_(p[:,:,kx,ky])
            p *= singular_value(p)

    @torch.no_grad()
    def step(self):

        G = 0
        for p in self.net.parameters():
            G += singular_value(p) * p.grad.norm(dim=(0,1)).sum()
        G /= self.depth

        log = math.log(0.5 * (1 + math.sqrt(1 + 4*G)))

        for p in self.net.parameters():
            factor = singular_value(p) / p.grad.norm(dim=(0,1), keepdim=True)
            p -= self.gain * log / self.depth * factor * p.grad
\end{minted}

\end{document}